\documentclass{article}
\usepackage{nips14submit_e,times}

\usepackage{amsfonts,amsmath,amssymb,amsthm}
\usepackage{url,multirow}
\usepackage{bm}
\usepackage{stmaryrd}
\usepackage{wrapfig}

\usepackage[linesnumbered,vlined,algoruled,titlenumbered]{algorithm2e}

\usepackage{units}
\usepackage{comment}

\usepackage{graphicx}
\usepackage{color}
\usepackage{epstopdf}
\usepackage{units}

\newcommand{\rank}{\mathrm{rank}}
\newcommand{\nhalf}{{\nicefrac{1}{2}}} 

\newcommand{\norm}[2]{\lVert{#1}\rVert_{#2}}

\newcommand{\mc}{\!\!:\!\!}

\newcommand{\R}{\mathbb{R}}  

\newcommand{\vect}[1]{\mathbf{#1}}
\newcommand{\vzero}{\vect{0}}

\newcommand{\vd}{\vect{d}}

\newcommand{\vn}{\vect{n}}

\newcommand{\vq}{\vect{q}}

\newcommand{\vu}{\vect{u}}

\newcommand{\vw}{\vect{w}}

\newcommand{\vy}{\vect{y}}
\newcommand{\vz}{\vect{z}}

\newcommand{\vth}{\vect{\boldsymbol{\theta}}}

\newcommand{\vnu}{\vect{\boldsymbol{\nu}}}

\newcommand{\verr}{\vect{\boldsymbol{\epsilon}}}

\newcommand{\ncw}{p}  
\newcommand{\dd}{n}	  
\newcommand{\Dict}{D} 
\newcommand{\at}{\vd} 
\newcommand{\atm}{\at_{\max}} 
\newcommand{\wv}{\vw} 
\newcommand{\we}{w} 
\newcommand{\tv}{\vy} 

\newcommand{\lm}{{\lambda_{\max}}}

\newcommand{\wvo}{\hat{\wv}} 
\newcommand{\weo}{\hat{w}} 
\newcommand{\vtho}{\hat{\vth}} 
\newcommand{\vtht}{\tilde{\vth}} 

\newcommand{\FS}{\mathcal F} 
\newcommand{\RR}{\mathcal{R}} 

\newcommand{\fr}{{\color{black}{\psi}}} 

\newcommand{\lat}{\lambda_{t}}

\newcommand{\fracn}[2]{{#1}/{#2}}

\newtheorem{proposition}{Proposition}
\newtheorem{theorem}{Theorem}
\newtheorem{lemma}{Lemma}
\newtheorem{corollary}{Corollary}

\DeclareMathOperator*{\sgn}{sign}
\DeclareMathOperator*{\diam}{diam}



\title{Feedback-Controlled Sequential Lasso Screening}

\author{Yun Wang, Xu Chen and Peter J.~Ramadge\\
Department of Electrical Engineering\\
Princeton University 
}

\nipsfinalcopy 

\begin{document}


\maketitle

\begin{abstract}
One way to solve lasso problems when the dictionary does not fit into available memory is to first screen the dictionary to remove unneeded features. Prior research has shown that sequential screening methods offer the greatest promise in this endeavor. Most existing work on sequential screening targets the context of tuning parameter selection, where one screens and solves a sequence of $N$ lasso problems with a fixed grid of geometrically spaced regularization parameters. In contrast, we focus on the scenario where a target regularization parameter has already been chosen via cross-validated model selection, and we then need to solve many lasso instances using this fixed value. In this context, we propose and explore a feedback controlled sequential screening scheme.
Feedback is used at each iteration to select the next problem to be solved.
This allows the sequence of problems to be adapted to the instance presented and the number of intermediate problems to be automatically selected. We demonstrate our feedback scheme using several datasets including a dictionary of approximate size 100,000 by 300,000.

\end{abstract}

\section{Introduction}\label{sec:intro}
Sparse dictionary-based representation of data
has proved effective in a wide
range of applications in computer vision, machine learning,
signal processing and statistics \cite{Elad2010Sparse, Wright2010Sparse}.
A sparse representation uses a dictionary
$\Dict\in\R^{\dd\times\ncw}$, with columns $\{\at_i\}_{i=1}^{\ncw}$
called \emph{features},
to represent a data point $\tv$ as a linear combination
$\sum_{i=1}^{\ncw} \we_{i}\at_{i}$
of a relatively small number of the features.
Thus for many $i$, we require $\we_i=0$.
One way to obtain such representations
is by solving the lasso problem \cite{RobTib1996}:
\begin{equation} \label{eq:lasso}
	\min_{\wv\in \R^{\ncw}} \qquad 	
	\nhalf \norm{\tv- \Dict\wv}{2}^2
	+ \lambda  \|\wv\|_{1},
\end{equation}
where $\lambda >0$ is a regularization parameter.

In many of today's sparse coding applications, the dictionary
can be very large, both because the data  dimension is large
(\cite{Wright2009Robust,Wagner2011Towards, Prasad2011Emerging, DeepSparseCoding})
and because an adequate representation requires many features (\cite{Wright2009Robust, Panagakis2009,Mallat2011,SRCvSLS2013, DeepSparseCoding}).
For example,  in music genre classification
the authors of \cite{XuChen2013} employed scattering representations \cite{Mallat2011, AndenMallat2013} of short segments of music data, resulting in dictionaries of size
$\dd=199$ and $\ncw=12,000$ (first order scattering),
and $\dd=11,726$ and $\ncw=12,000$ (second order scattering).
Once dictionaries reach this size, solving \eqref{eq:lasso}
can become a bottleneck in the overall computation.
More importantly, for significantly larger dictionaries, fitting the dictionary
into available memory is a concern.

Several approaches are possible for solving  \eqref{eq:lasso} with a large dictionary. For example, assuming the dictionary fits into memory, one could use early termination of an iterative solver (e.g., FISTA \cite{Beck2009}) to quickly obtain an approximate solution.
Alternatively, one could use a finite number of steps of a sequential greedy method (e.g., OMP \cite{TroppGilbert2007}) to approximate a solution of \eqref{eq:lasso}.
OMP iteratively selects a feature with the largest correlation with the current residual.
Searching for this feature is the most time consuming part of the algorithm but only requires part of the dictionary to be loaded at once. To obtain a solution $\wvo$ that is $m$-sparse, the algorithm needs to hold at most $m$ features in the memory.
However, since the true sparsity is unknown, in general, OMP and similar algorithms will only yield an approximate solution.
Both of these approaches lead (in general) to an approximate solution of \eqref{eq:lasso}.

Several recent studies have investigated the following alternative approach.
For a given $\tv$ and $\lambda$, use duality to quickly identify a subset of
features in $\Dict$ that are certified to have zero weight in a solution $\wvo$ of \eqref{eq:lasso}. By removing these features from the dictionary and solving a smaller lasso problem, we can obtain an exact solution to the original problem.
In \cite{Ghaoui2012} this approach is called SAFE screening,
indicating that screening does not reject any needed feature.
Just as important, screening can be executed with only a few features loaded into memory at once. Hence it has the potential to significantly reduce the size of dictionary that is provided to a lasso solver.
Recent work on safe screening  has further developed this idea \cite{Xiang2011Learning_b,Xiang2012Fast,
ellpscr2012,dpp2015,YunWang2013a, GramfortICML15, GramfortNIPS15, GramfortNIPS16}.
The survey \cite{LSS} reviews these recent methods.
A related but distinct form of lasso screening \cite{Tibshirani2010Strong} allows false rejection of features with small probability. We do not explicitly consider such screening tests, but our results are also relevant to this approach.

The core idea of SAFE screening is to bound the solution of the dual problem
of \eqref{eq:lasso} within a region $\RR$ and compute
$\mu(\at_i)=\max_{\vth\in \RR} \at_i^T\vth$.
The value $\mu({\at_i})$ is then used to decide if feature
$\at_i$ can be removed.
Since these tests are applied once prior to solving \eqref{eq:lasso},
we refer to them as ``one-shot'' tests.
A parameter that is often useful in such tests is the value $\lm=\max_i |\at_i^T\tv|$.
Roughly, this gives a measure of how well a feature in $\Dict$ matches $\tv$.
Empirical studies indicate that current one-shot screening tests perform well
for normalized values of $\lambda/\lm$ of moderate to large size,
but performance quickly declines as $\lambda/\lm$ falls below $0.4$ \cite{LSS}.
To address this problem, one can employ a more complex form of screening known as sequential screening \cite{Ghaoui2012, Tibshirani2010Strong}.
In sequential schemes, one screens and solves a sequence of $N$ lasso problems with $\lambda_1>\cdots>\lambda_N$.
The key point is that for instance $(\tv,\lambda_k)$, one utilizes the dual solution $\vtho_{k-1}$ to obtain a region bound on $\vtho_k$. This is used to apply a one-shot screening test to reduce the dictionary. And a standard solver is then used to solve the reduced problem to find  $\wvo_k$.

Existing state-of-the-art sequential screening algorithms include sequential Dome \cite{YunWang2013b}, sequential Strong rule \cite{Tibshirani2010Strong}, sequential Enhanced DPP rule \cite{dpp2015} among others \cite{Ghaoui2012}. They are situated in the context of model selection. In this setting, all solutions to the sequence of problems are of interest, and the sequence of regularization parameters is thus fixed a priori (typically taking a log-grid of values). Sequential screening is then conducted along this predetermined sequence to expedite this parameter tuning process.

In contrast, we focus on another distinct application context where the best regularization parameter, denoted by $\lambda_t$, has already been chosen via cross-validated model selection, and then we need to solve many lasso instances using this fixed value. Of our particular interest, is $\lambda_t/\lm$ around $0.1$ since many applications of the sparse representation framework have found this range to be most helpful. For example, \cite{Xiang2011Learning_b} found that $\lat/\lm$ in the range $[0.1,0.2]$ maximizes SVM classification accuracy for the handwritten digit recognition problem. In music genre classification, the authors of \cite{XuChen2013} selected $\lat/\lm$ in the range $[0.1,0.15]$ via cross-validation to maximize classification accuracy using the SRC classifier \cite{Wright2009Robust}. Once the proper $\lat$ to use is known, the problem of efficient sequential screening can be very different from the sequential screening that is targeted for the model selection purpose: the solution at $\lat$ is the only one that is of interest, and all other problem instances are merely way points. We are thus given the freedom to design the sequence of regularization parameters $\{\lambda_k\}_{k=1}^N$ with $\lambda_N=\lat$ that specifically target a single problem instance at $\lat$.
In this scenario, the open-loop sequential screening schemes using a geometrically spaced sequence $\{\lambda_k\}_{k=1}^N$, which is determined before any problem in the sequence is solved, may not be a good idea. And it would be an advantage if $\{\lambda_k\}_{k=1}^N$ could be tuned individually to each instance $(\Dict, \tv,\lat)$.


This paper makes two contributions. First, we design a feedback mechanism to adaptively select the sequence  $\{\lambda_k\}_{k=1}^N$ and $N$ for sequential screening.
The feedback mechanism automatically selects the next value
$\lambda_k$ as a function of the results seen in previous steps of the
sequential screening process. It also determines when to stop, and hence
automatically selects $N$. We call our feedback controlled sequential screening scheme, data-adaptive sequential screening (DASS).
DASS has the advantage that $N$ and $\{\lambda_k\}_{k=1}^N$ are automatically
adapted during the screening  process to the particular instance $(D, \tv,\lat)$.
In addition, the feedback mechanism bounds the region diameter
used for one-shot screening at each iteration by a value set by the user.
Second, we examine the effects of the inevitable errors that accrue in obtaining the solutions (and dual solutions) of the intermediate lasso problems during sequential screening. At each step, sequential screening assumes exact knowledge of the previous dual solution. However, practical lasso solvers introduce a small error.
In the context of classification, we show that the DASS scheme is robust to these errors.

We give required background in \S\ref{sec:prelim}.
Then we introduce DASS and examine its properties in \S\ref{sec:dass}.
We show that DASS ensures the diameters of the regions used for the one-shot screening are bounded by a user selected value, and the number of intermediate lasso problems is also bounded.
We then address the issue of inaccuracies in lasso solutions.
\S\ref{sec:exp} discusses the performance of DASS on a selection of datasets, and we conclude in \S\ref{sec:conclude}. Proofs are given in the supplementary material.

\section{Background}\label{sec:prelim}
The one-shot region tests discussed in the introduction
are based on the Lagrangian dual of \eqref{eq:lasso}:
\begin{equation}
	\begin{split}
		\label{eq:dual}
		\max_{\vth\in \R^{\dd}}  \qquad
		& \nhalf \norm{\tv}{2}^2-
		\nicefrac{\lambda^2}{2}
		\norm{\vth -\tv/\lambda}{2}^2\\
		\text{s.t.} \qquad
		& |\at_i^T\vth| \leq 1 \quad \forall
		i=1,2,\ldots,\ncw,
	\end{split}
\end{equation}
with the solutions $\wvo\in\R^{\ncw}$ of
\eqref{eq:lasso} and $\vtho\in \R^{\dd}$ of \eqref{eq:dual} related by:
\begin{equation}\label{eq:relationship}
	\tv = \Dict \wvo + \lambda \vtho,
		\qquad \at_i^T\vtho =
		\begin{cases}
		  \sgn{\weo_i}, &\text{ if } \weo_i\neq 0;\\
		  \gamma\in \left[-1,1\right] & \text{ if } \weo_i=0.
		\end{cases}
\end{equation}
This dual problem has been extensively discussed in the literature, \cite{LSS,Xiang2011Learning_b}.
The set of feasible points $\FS$ of the dual problem is the nonempty, closed, convex polyhedron formed by the intersection of the finite set of closed half spaces $\at^T\vth\leq 1$ for
$\at\in \{\pm \at_i\}_{i=1}^\ncw$.
The objective in \eqref{eq:dual} is maximized by
the projection $\vtho$ of $\tv/\lambda$ onto $\FS$.
This is the \emph{unique} point satisfying
\cite[\S3.1]{HU2001}:
\begin{equation}\label{eq:projineq}
(\tv/\lambda-\vtho)^{T}(\vth-\vtho) \leq 0,
\quad \text{for each } \vth\in \FS.
\end{equation}
\begin{wrapfigure}{r}{0.425\textwidth}
\centering
\includegraphics[width=0.4\textwidth]{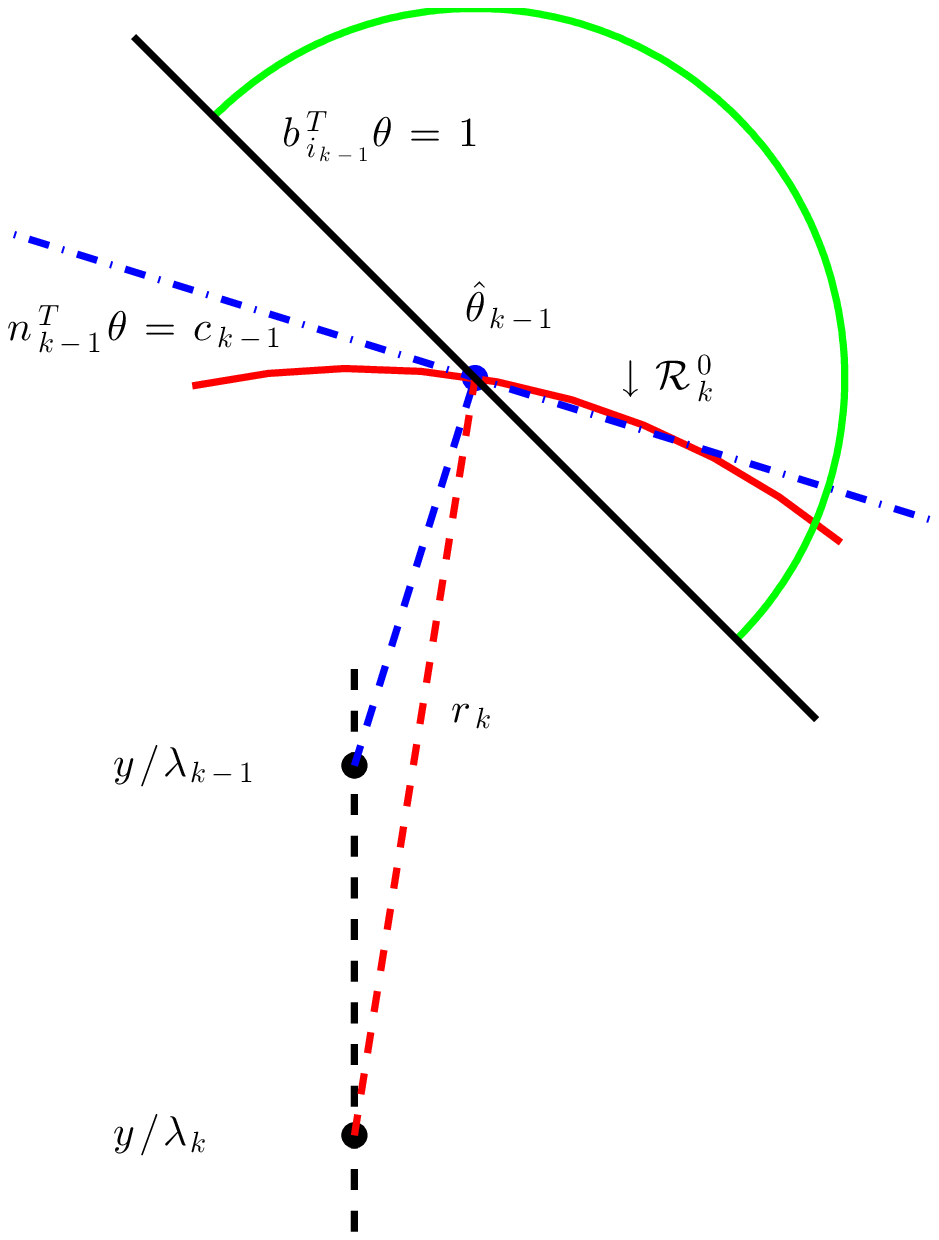}
\caption{\small
Sequential screening.
At step $k$, $\vtho_{k-1}$ defines a spherical bound (solid red curve) on $\vtho_k$
(eq. \eqref{eq:spk-1}).
The (blue dash-dot) hyperplane $\vn_{k-1}^T \vth \leq c_{k-1}$ (eq. \eqref{eq:nck-1})
separates $\FS$ from $\tv/\lambda_{k-1}$.
Hence $\vtho_k$ is contained in the dome $\RR_k^0$ (eq. \eqref{eq:Rk}) formed within the sphere and above the hyperplane.
The solid line (black) is the hyperplane of an active feature at $\vtho_{k-1}$.
The large hemisphere (green) is the DPP bound $\|\vtho_k -\vtho_{k-1}\|_2
\leq 1/\lambda_k - 1/\lambda_{k-1}$ used in \cite{dpp2015}.
The diagram is drawn to scale.
 }
 \label{fig:boundonth}
 \vspace{-1cm}
\end{wrapfigure}
Let $\lm = \max_{i=1}^p |\at_i^T\tv|$.
%
%
For $\lambda\geq \lm$, $\vtho(\lambda)=\tv/\lambda \in \FS$,
with $\vtho(\lm)=\tv/\lm$  on the boundary of $\FS$.
For $\lambda < \lm$, $\tv/\lambda$ moves away from $\FS$ and
$\vtho(\lambda)$ is its unique projection onto the boundary of $\FS$.

To form a one-shot
test, one bounds the dual solution $\vtho$ in a compact set $\RR$,
and computes $\mu(\at_i) =\max_{\vth\in \RR} \at_i^T\vth$.
By \eqref{eq:relationship}, if a feature $\at_i$ satisfies $\mu(\at_i)<1$ and $\mu(-\at_i)<1$,
then $\weo_i=0$; see \cite{Ghaoui2012, LSS}.
The bounding region $\RR$ can be selected as the intersection of a sphere $S(\vq,r) = \{\vth: \|\vth-\vq\|_2 \leq r\}$,
with center $\vq$ and radius $r$, and $m$ half spaces, i.e., $\RR=\{\vth: \|\vth-\vq\|_2 \leq r, \vn_i^T\vth\leq c_i, i=1,\dots,m\}$.
Then $\mu(\at_i) =\max_{\vth\in \RR} \at_i^T \vth$ is obtained by solving a convex program, with closed form solutions available for $m=0,1,2$ \cite{LSS, SymmetryOptimization}.
%

Sequential screening screens and solves a sequence of problems with $\lambda_1>\ldots>\lambda_N$ until $\lambda_N=\lat$. In an open-loop scheme, $N$ and $\{\lambda_k\}_{k=1}^N$ are selected before any solution is obtained.
For example, fix $N$, set $\lambda_1=0.95 \lm$, and
space the $\lambda_k$ geometrically:  $\lambda_k =\alpha \lambda_{k-1}$
with $\alpha = (\lat/\lambda_1)^{1/(N-1)}$.


Note that $\vtho_1$ must lie in the sphere $S(\vq_1,r_1)$  with center  $\vq_1 = \tv/\lambda_1$ and radius $r_1=\|\tv/\lambda_1 - \tv/\lm\|_2$ \cite{LSS}.
In addition,
if $\atm$ is a feature satisfying  $|\atm^T \tv|=\lm$,
then  $\vtho_1$ lies in the half space $H_{\max}=\{\vth\colon\atm^T \vth \leq 1\}$.
So to screen the dictionary at $k=1$, we use the one-shot screening test for
\begin{equation}\label{eq:R1}
\RR_1^0=S(\vq_1,r_1) \cap H_{\max} .
\end{equation}
Then we solve the reduced lasso problem to obtain
$\vtho_1$ (see \eqref{eq:relationship}) .

At step $k$, we make use of the dual solution $\vtho_{k-1}$ for instance $(\tv, \lambda_{k-1})$.
Since $\vtho_{k-1}$ is dual feasible, $\vtho_k$  is contained in the sphere $S(\vq_k,r_k)$ with
\begin{equation}\label{eq:spk-1}
\vq_k=\tv/\lambda_k, \qquad
r_k=\|\tv/\lambda_k - \vtho_{k-1}\|_2.
\end{equation}
In addition, since $\vtho_{k-1}$ is the projection of
$\tv/\lambda_{k-1}$ onto $\FS$,
by \eqref{eq:projineq} we have
$\forall \vth\in \FS, (\tv/\lambda_{k-1}- \vtho_{k-1})^{T}\vth
\leq (\tv/\lambda_{k-1}-\vtho_{k-1})^{T}\vtho_{k-1}.
$
Since $0\in \FS$, the right hand side of this inequality is nonnegative.
Hence
$\FS$ lies in the closed half space
$H_k=\{\vth\colon \vn_{k-1}^{T} \vth \leq c_{k-1}\}$ with
\begin{equation}\label{eq:nck-1}
\vn_{k-1}=\frac{\tv/\lambda_{k-1}-\vtho_{k-1}}
{\|\tv/\lambda_{k-1} -\vtho_{k-1}\|_{2}} ,
\qquad
c_{k-1}=\vn_{k-1}^{T}\vtho_{k-1}.
\end{equation}
In particular, $\vtho_k\in \FS$,
so $\vtho_k$ is contained in the region
\begin{equation}\label{eq:Rk}
\RR^0_k=S(\vq_k,r_k)\cap H_k .
\end{equation}
This is illustrated in Fig.~\ref{fig:boundonth}.
We now use the one-shot screening test for $\RR_k^0$
to screen the dictionary.
Then solve the reduced lasso problem  to obtain $\vtho_k$.

\section{Feedback Controlled Sequential Screening}\label{sec:dass}

%
%
Under the assumption of exact computation, we now develop our sequential screening feedback rule and analyze its properties. The issue of approximate lasso solution is examined later.
The diameter of a set $\RR$, denoted by $\diam(\RR)$,
is the maximum distance between any two points in $\RR$.
The following preliminary lemma will be useful.

\begin{lemma}\label{lem:diamR}
If $\RR=\{\vth: \|\vth-\vq\|_2\leq r, \vn^T\vth\leq c\}$ has a nonempty interior, then
$$
\diam(\RR) =
\begin{cases}
2\sqrt{r^2-(\vn^T\vq-c)^2}, & \text{if } \vq \notin \RR;\\
2r, & \text{otherwise}.
\end{cases}
$$
\end{lemma}

We next show that if $S(\vq_k, r_k)$, $H_k$ and $\RR^0_k$ are constructed via \eqref{eq:spk-1}, \eqref{eq:nck-1}, and \eqref{eq:Rk}, then $\vq_k \notin \RR^0_k$.

\begin{lemma}\label{lem:q}
For $\RR^0_k$ constructed via \eqref{eq:spk-1}, \eqref{eq:nck-1}, and \eqref{eq:Rk},
$\vq_k\notin \RR^0_k$.
\end{lemma}

Combining Lemmas \ref{lem:diamR} and  \ref{lem:q} yields the following result.

\begin{proposition} \label{pro:diamRRk}
For $\RR^0_k$ constructed via \eqref{eq:spk-1}, \eqref{eq:nck-1}, and \eqref{eq:Rk},
\begin{equation}\label{eq:diam}
\diam(\RR^0_k) = 2\left (
\frac{1}{\lambda_k}-\frac{1}{\lambda_{k-1}}
\right)
\left( \tv^T(I-\vn_{k-1}\vn_{k-1}^T)\tv \right)^{1/2}.
\end{equation}
\end{proposition}


The region $\RR^0_k$ provides an initial bound for the one-shot test applied at step $k$.
More half space constraints can be used to form a final bounding region $\RR_k$.
For example, the one-shot test THT in \cite{YunWang2013a, LSS, WangThesis2015} does this by adding a second appropriately selected half space constraint from the dual problem.
No matter how, or how many, additional half space
constraints are added, \eqref{eq:diam} provides an upper bound on the diameter of the final bounding region.
Thus we have the following corollary.

\begin{corollary}
If $\RR_k$ is obtained by adding a finite number of additional half space constraints to $\RR^0_k$, then $\diam(\RR_k)$ is bounded above by the expression in \eqref{eq:diam}.
\end{corollary}


Proposition \ref{pro:diamRRk} suggests a rule for
selecting $\lambda_k$ so that
$\diam(\RR_k) \leq R$ where $R>0$ is a user specified parameter.
Using \eqref{eq:diam}, this is achieved by selecting $\lambda_k$ so that:
\begin{align}\label{eq:updater}
\frac{1}{\lambda_k}
=  \frac{1}{\lambda_{k-1} }+
\frac{\nhalf R}{\sqrt{\tv^T(I-\vn_{k-1}\vn_{k-1}^T)\tv} }
\end{align}
\begin{wrapfigure}{L}{0.47\textwidth}
\centering
\IncMargin{1em}
\begin{algorithm}[H]
\Indm
\KwIn{A lasso instance $(\Dict, \tv, \lat)$, and scalar $R>0$}
\KwOut{Solution $\wvo_t$ of the lasso instance}
\Indp
\BlankLine
Initialize $N=0, k=1$, and $\lambda_1$ with $\lat <\lambda_{1}<\lm$\;
\If{k=1}{
Use region \eqref{eq:R1} to one-shot screen $D$ for lasso instance $(\tv,\lambda_1)$\;
Solve the reduced lasso for $\wvo_1$ and $\vtho_1$\;
}
\BlankLine
\While{$\lambda_k>\lat$}{
Set k=k+1 and use \eqref{eq:updater} to obtain $\lambda_k$\;
\If{$\lambda_k\leq \lat$}{
Set $N=k$ and $\lambda_k=\lat$\;
}
\BlankLine
Use $\vtho_{k-1}$ and region \eqref{eq:Rk} to one-shot screen $D$ for lasso instance $(\tv,\lambda_k)$\;
Solve the reduced lasso for $\wvo_k$ and $\vtho_k$\;
}
\BlankLine
\Return{$\wvo_N$}\;
\caption{DASS}\label{alg:DASS}
\end{algorithm}
\DecMargin{1em}
\vspace{-0.85cm}
\end{wrapfigure}

Assuming $\tv$ is not aligned with $\vtho_{k-1}$ (this holds generically),
$\tv^T(I-\vn_{k-1}\vn_{k-1}^T)\tv>0$. Incorporating \eqref{eq:updater} into a sequential screening scheme yields the DASS algorithm in Algorithm \ref{alg:DASS}.
We can initialize $\lambda_1=0.95 \lm$.
At each step, we first use one-shot screening  based on the region $\RR^0_k$,
then solve the resulting reduced lasso problem  (for $k>1$, $\wvo_{k-1}$
can be used as a warm start).
Hence each problem is solved efficiently.

If we write $s_k=1/\lambda_k$, then
\eqref{eq:updater} becomes
$
s_k -s_{k-1} = f(\vn_{k-1}^T \tv)
$.
Thus $s_k$ is a first order, autoregressive process driven
by a nonnegative, nonlinear function of $\vn_{k-1}^T\tv$.
Since $\vn_{k-1}$ depends on the solution of the previous instance, this term is providing feedback
in the update of $s_k$.
The feedback rule ensures that $\lambda_k$ decreases at each step;
this results in $\lambda_k<\lat$, and termination of the algorithm.
We show in Theorem \ref{thm:smlard} that the algorithm always terminates,
give an upper bound on the number of iterations required,
and verify that for $k>1$ the one-shot tests
use a region of diameter at most $R$.
\vspace{0.1cm}

\begin{theorem}\label{thm:smlard}
For $0<\lat<\lm$, let $N$ and the sequence $\{\lambda_k\}_{k=1}^N$ be selected
using the DASS algorithm.
The DASS sequence is guaranteed to terminate after a finite number of steps $N$,
and $\diam(\RR_k) \leq R$, $k=2\mc N$. In addition, if the dual regularization path
$\vtho(\lambda)$ is bounded, i.e., there exists $C>0$, such that $\|\vtho(\lambda)\|\leq C$ for all $\lambda >0$, then
$$
N \leq 1+\frac{\log(1/\lat)}{\log(1+R/2C)} \ .
$$
\end{theorem}

The assumption in Theorem \ref{thm:smlard} that $\vtho(\lambda)$ is bounded holds if $\tv$ is in the range of $D$ (supplementary material). It always holds if $\rank(D)=\dd$ \cite{PMXu2013}.
We show the number of iterations required for the algorithm to terminate in real problem instances in \S \ref{sec:exp}.




In practice,
it is more realistic to assume we have an approximation $\vtht_{k-1}$ to $\vtho_{k-1}$, with  $\|\vtho_{k-1} - \vtht_{k-1}\|_2\leq \rho$.
For simplicity we assume $\rho$ does not depend on $k$, but this is not essential.
We first verify that the DASS algorithm still terminates after a finite number of steps
and give a bound on the number of steps required.
We state this as a corollary to Theorem \ref{thm:smlard}.

\begin{corollary}\label{cor:smlard}
Suppose we have an approximate dual solution
${\vtht}_k$, with $\|\vtho_{k} - \vtht_{k}\|_2\leq \rho$, $k=1\mc N$.
Under this assumption, the DASS sequence terminates with $N \leq 1+\frac{\log(1/\lat)}{\log(1+ R/2(C+\rho))}$.
\end{corollary}

In many classification applications, the lasso regression is often used as a feature extraction step. For each data point $\tv$ to be classified, lasso regression is first used to compute the solution $\wvo$. Then as a feature vector, $\wvo$ is fed into a classifier, e.g., SVM, SRC. We analyze in this classification setting, the effect of accumulated inaccuracy in our sequential screening process. Consider using an approximate solution $\tilde{\wv}=\wvo+\vn$ for classification.
When used in sequential screening, $\tilde{\wv}_{k-1}$ gives an approximate
dual solution $\vtht_{k-1}$, which is used for one-shot screening at step $k$.
So the one-shot test may both fail to reject, and falsely reject features.
The error in $\tilde{\vw}_k$ thus has two sources: from the lasso solver itself, and
from false feature rejection. The latter errors can propagate, leading
DASS to fail to give adequate rejection, or to make too many false rejections causing classification performance to suffer.
We investigate this in \S \ref{sec:exp}.

\begin{figure}[t!]
\centering
\includegraphics[width=0.24\textwidth]{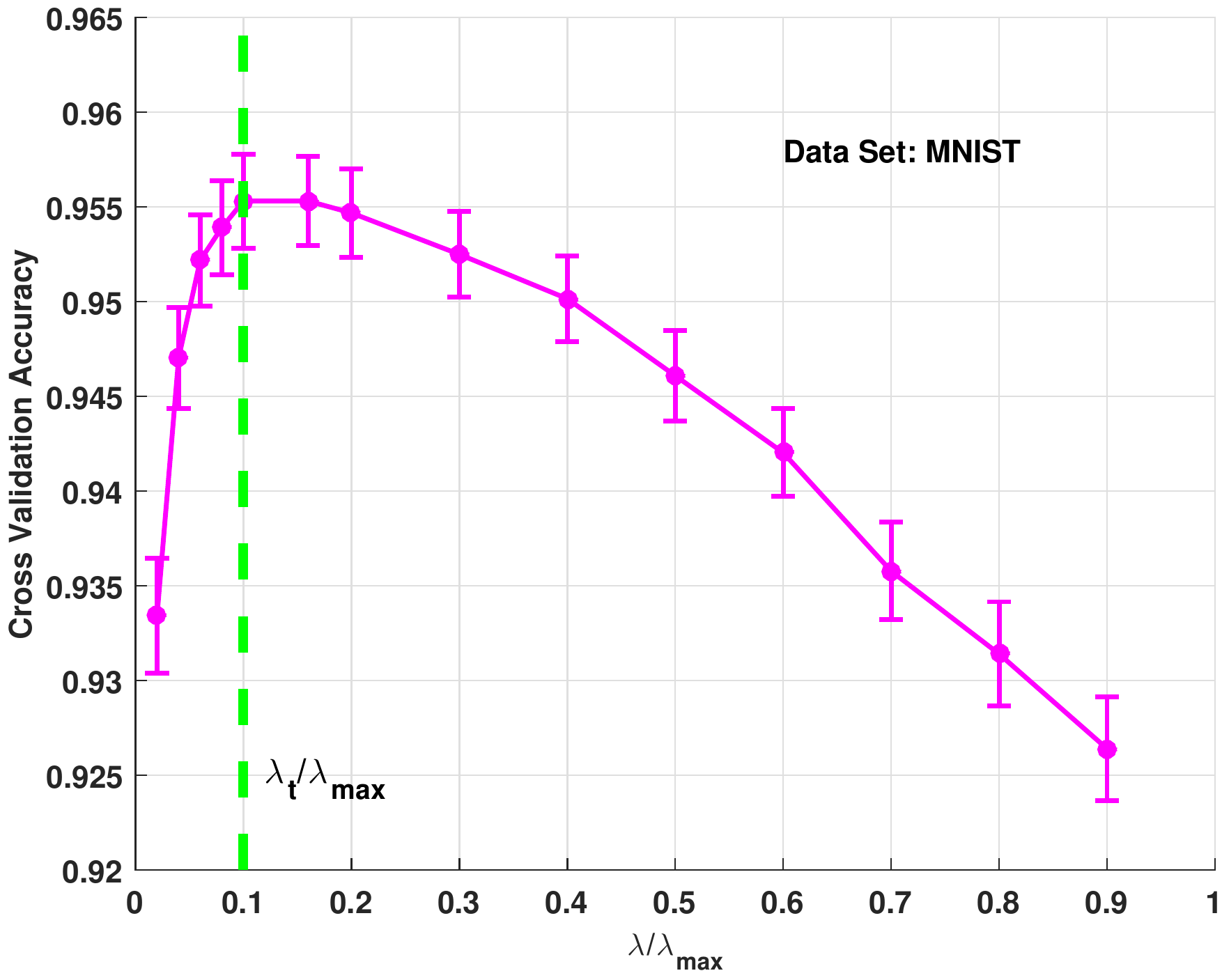}
\includegraphics[width=0.24\textwidth]{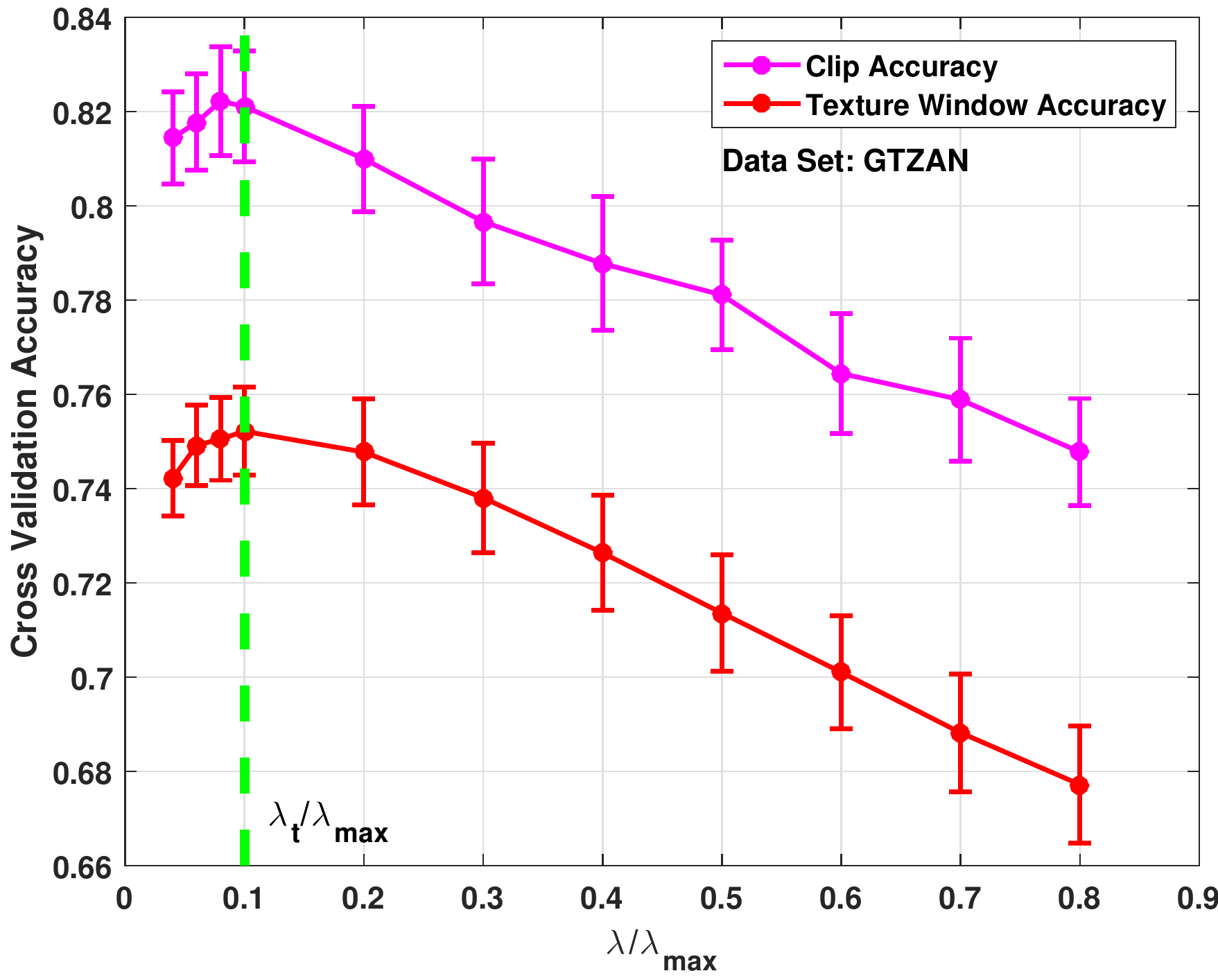}
\includegraphics[width=0.24\textwidth]{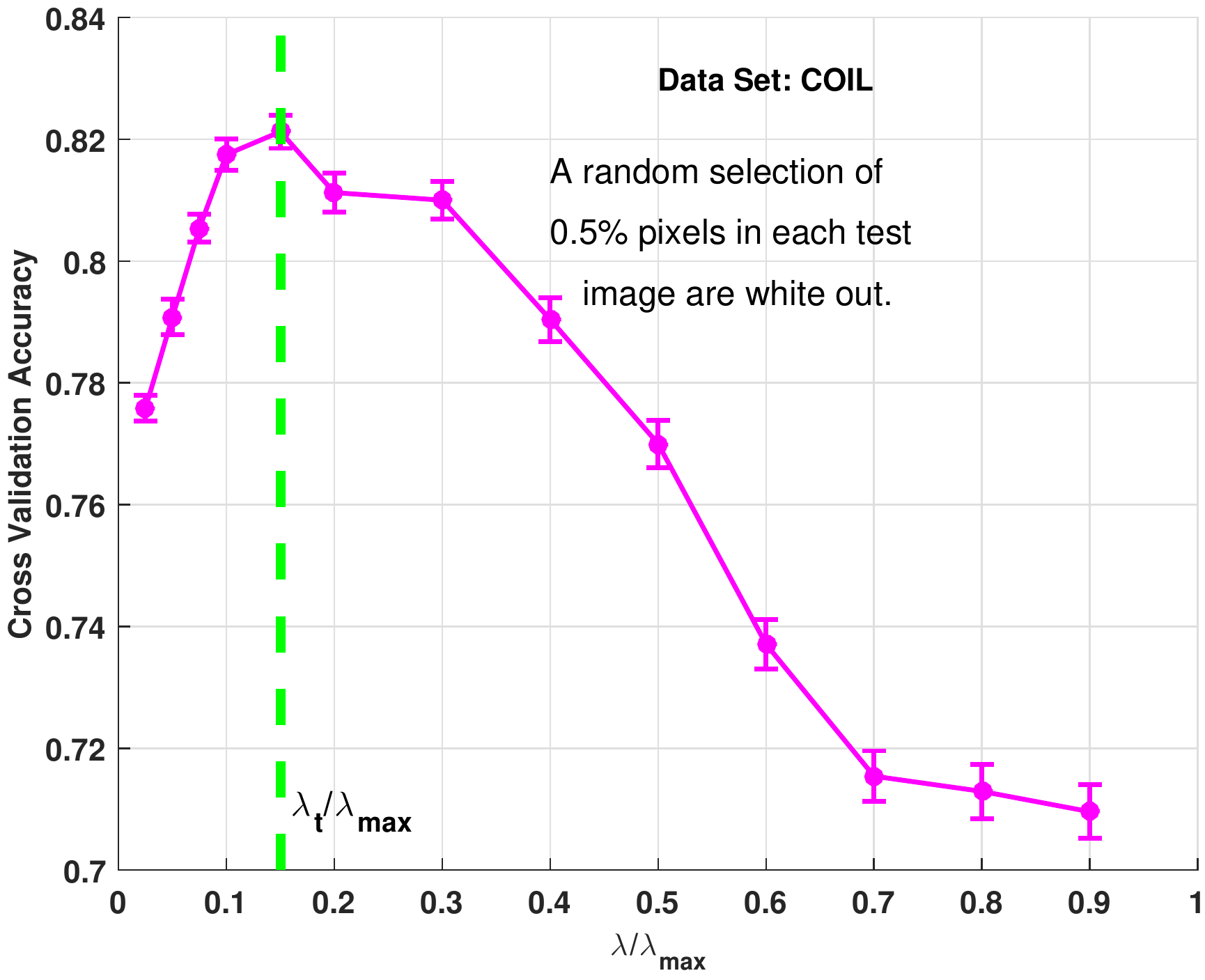}
\includegraphics[width=0.24\textwidth]{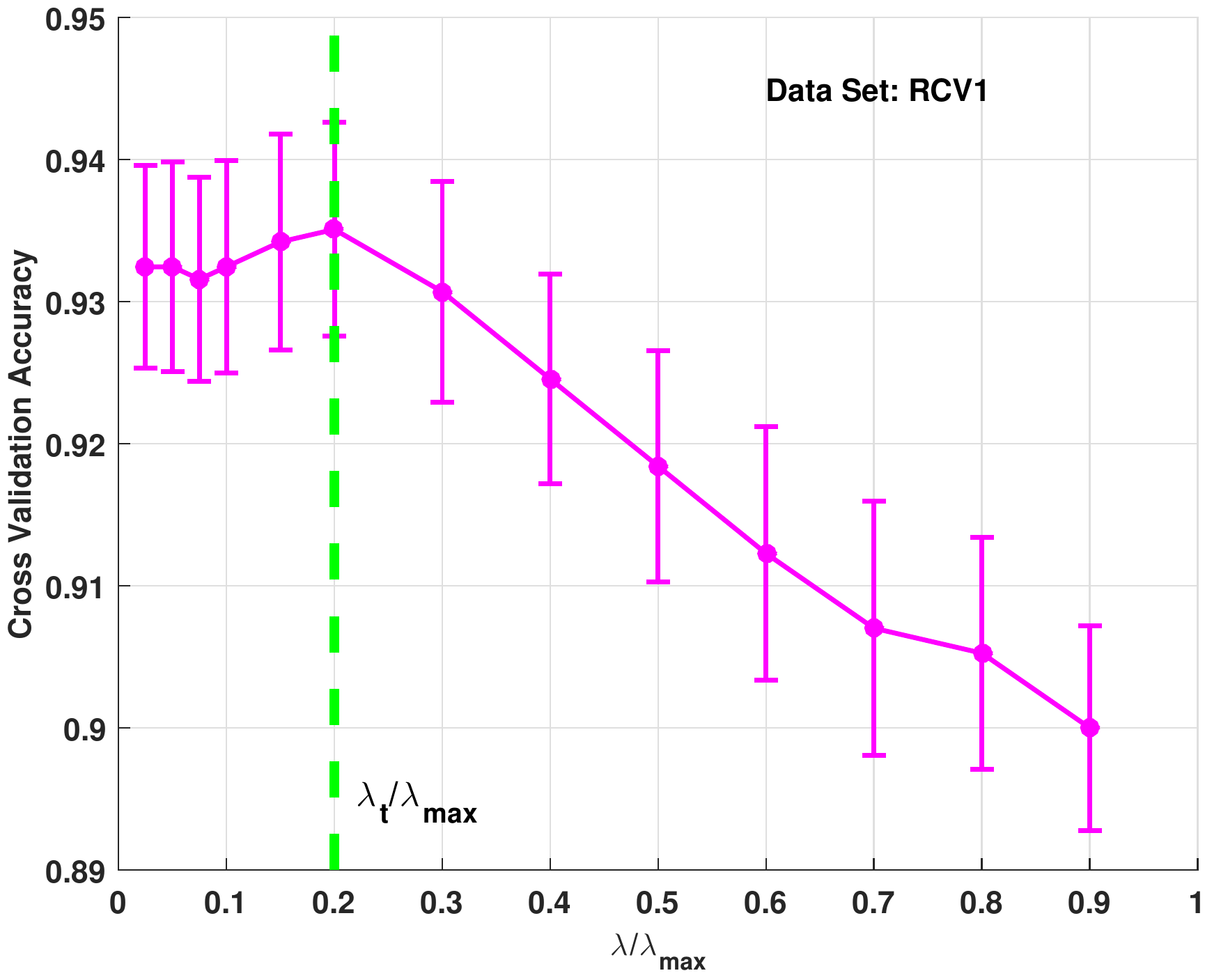}
\caption{\small
$\lambda_t/\lm \approx 0.1$ is of particular interest as cross validation shows this range is most helpful for many classification problems.
(From left to right): handwritten digit recognition on MNIST \cite{LeCun1998The-MNIST}, music genre classification on GTZAN \cite{Tzanetakis2002}, object recognition on COIL \cite{Nene1996Columbia} and text categorization on RCV1 \cite{RCV1}.
}
\label{fig:cv}
\end{figure}

\section{Empirical Tests}\label{sec:exp}
\begin{figure}[t!]
\centering
\includegraphics[width=0.245\textwidth]{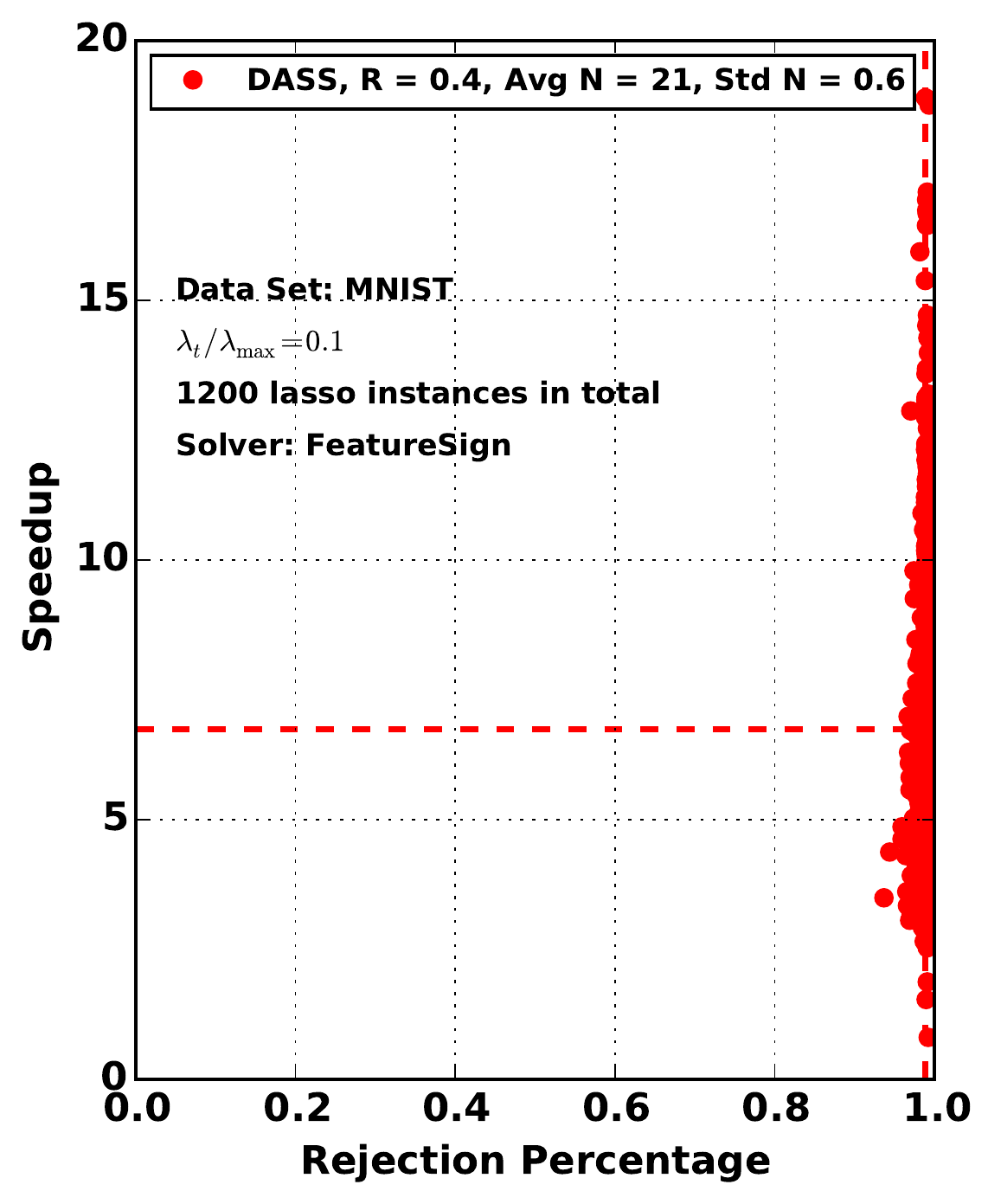}
\includegraphics[width=0.245\textwidth]{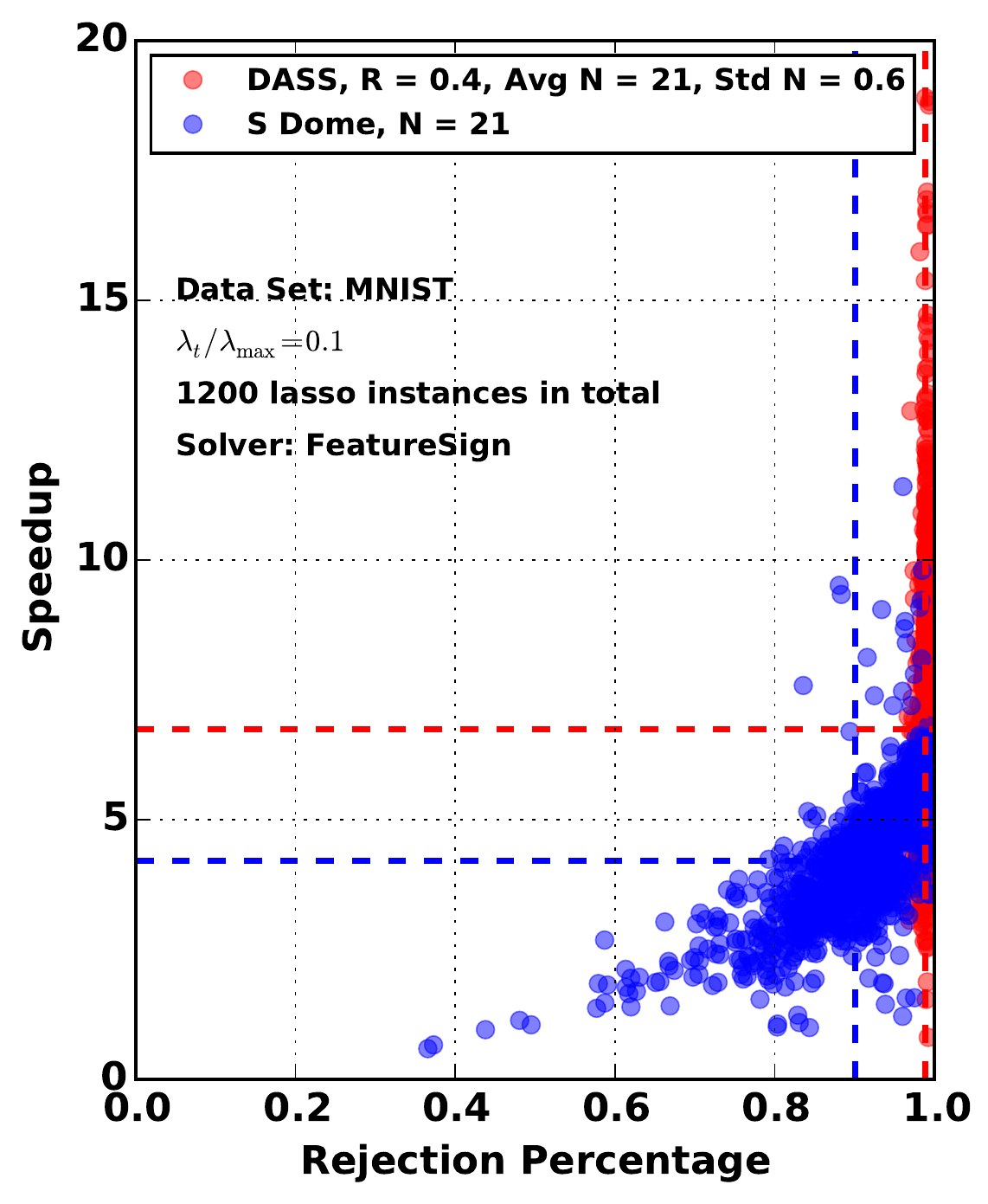}
\includegraphics[width=0.245\textwidth]{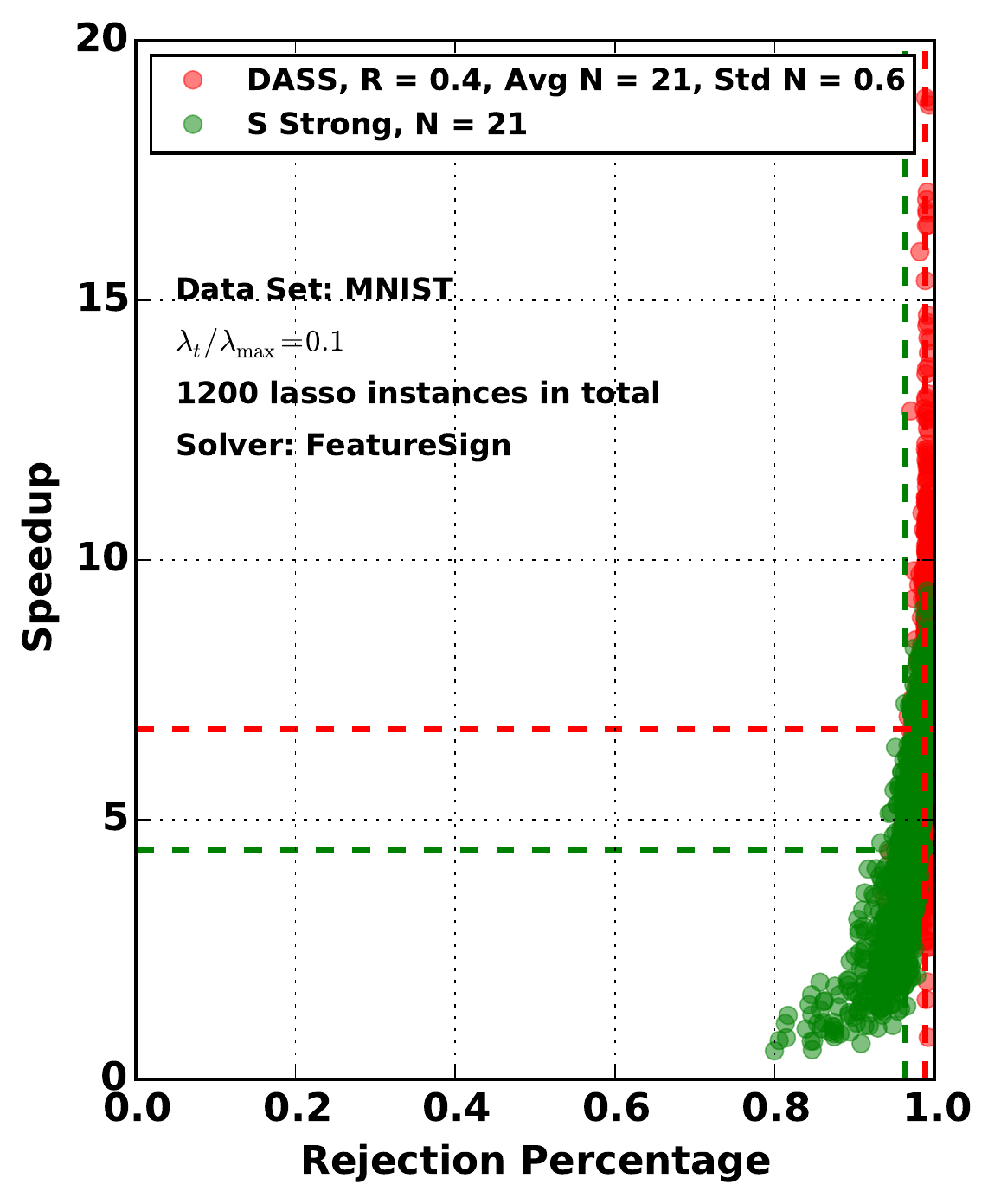}
\includegraphics[width=0.245\textwidth]{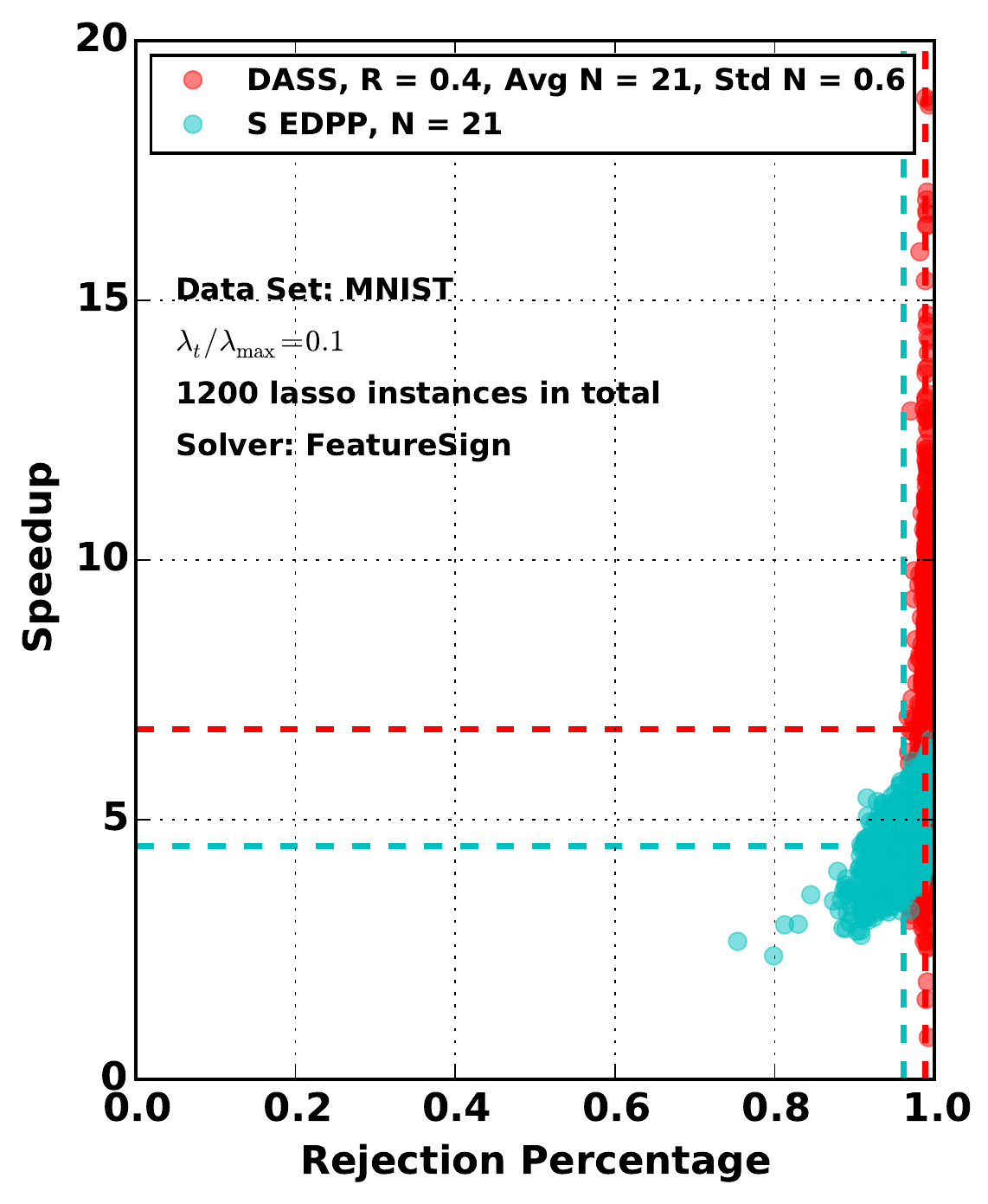}
\\
\includegraphics[width=0.49\textwidth]{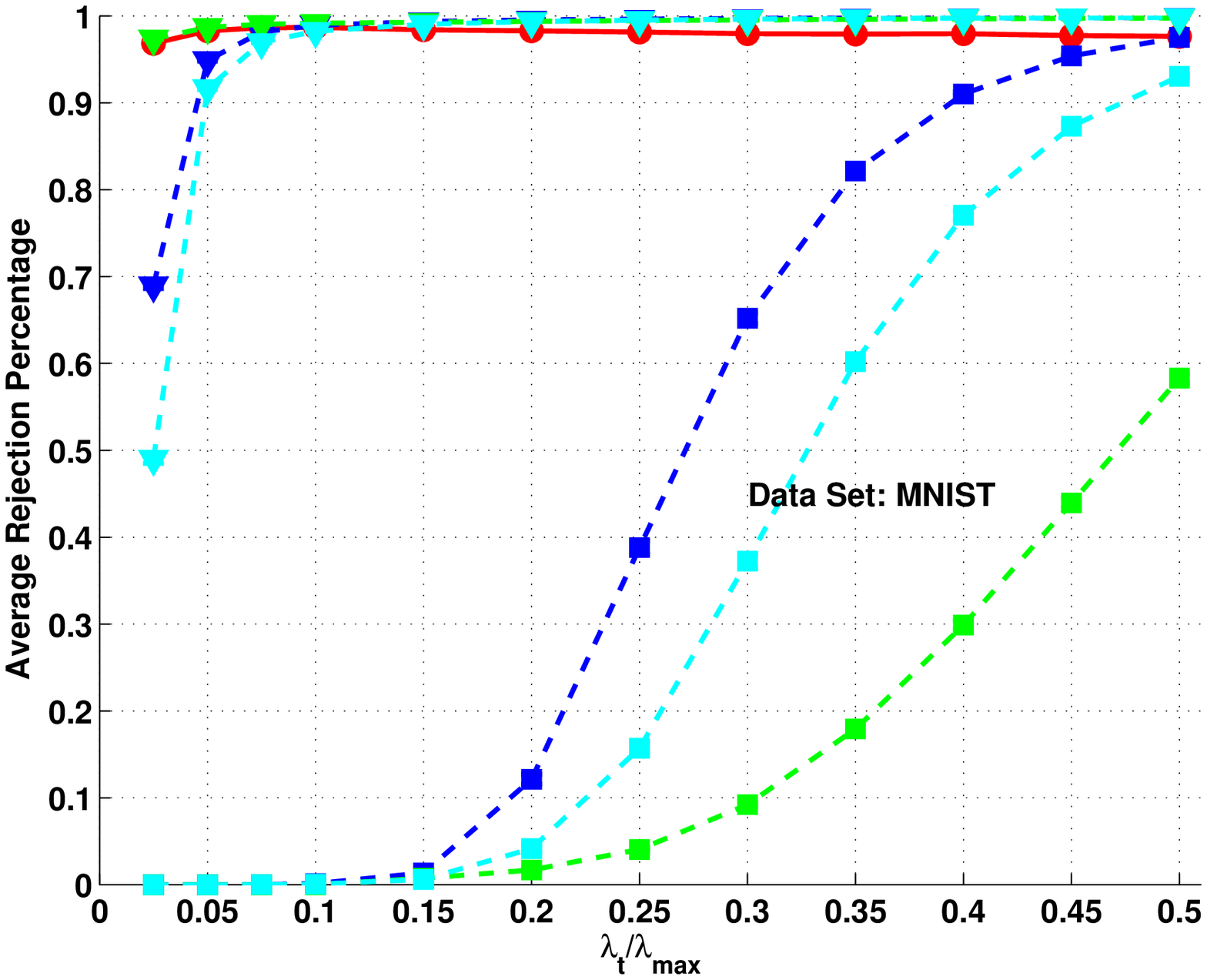}
\includegraphics[width=0.49\textwidth]{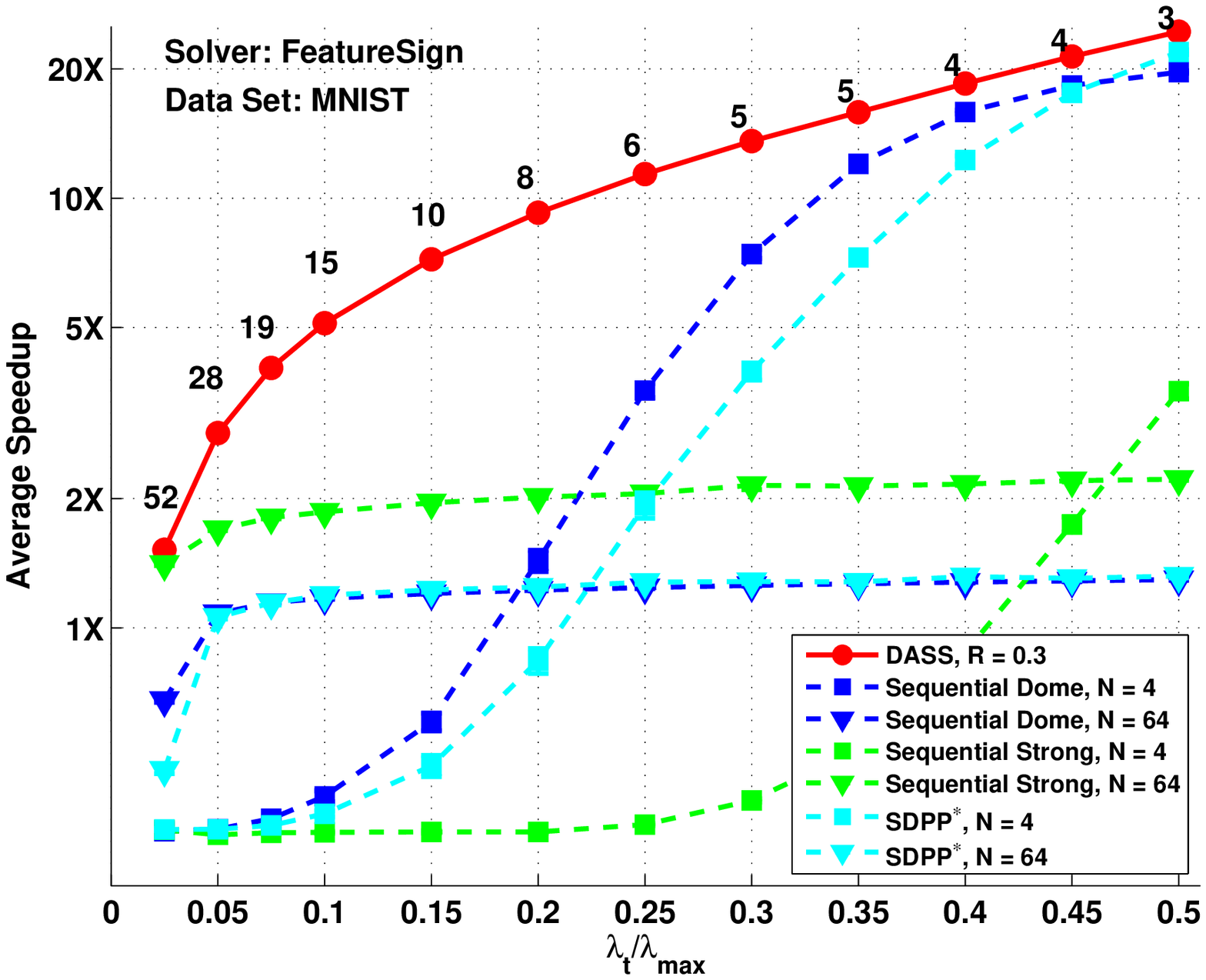}
\caption{\small
{\bf Comparison of DASS with sequential Dome, sequential Strong rule and sequential Enhanced DPP rule on MNIST.}
(Top): Scatter plot of speedup versus rejection percentage for lasso instances ($D, \tv$) with $\lat/\lm=0.1$.
(Bottom): Average rejection and speedup for
$\lat/\lm \in [0.025, 0.5]$. The average $N$ for DASS is shown on the speedup curve.
}
\vspace{-0.5cm}
\label{fig:MNIST}
\end{figure}
\begin{figure}[b!]
\vspace{-0.5cm}
\centering
\includegraphics[width=0.245\textwidth]{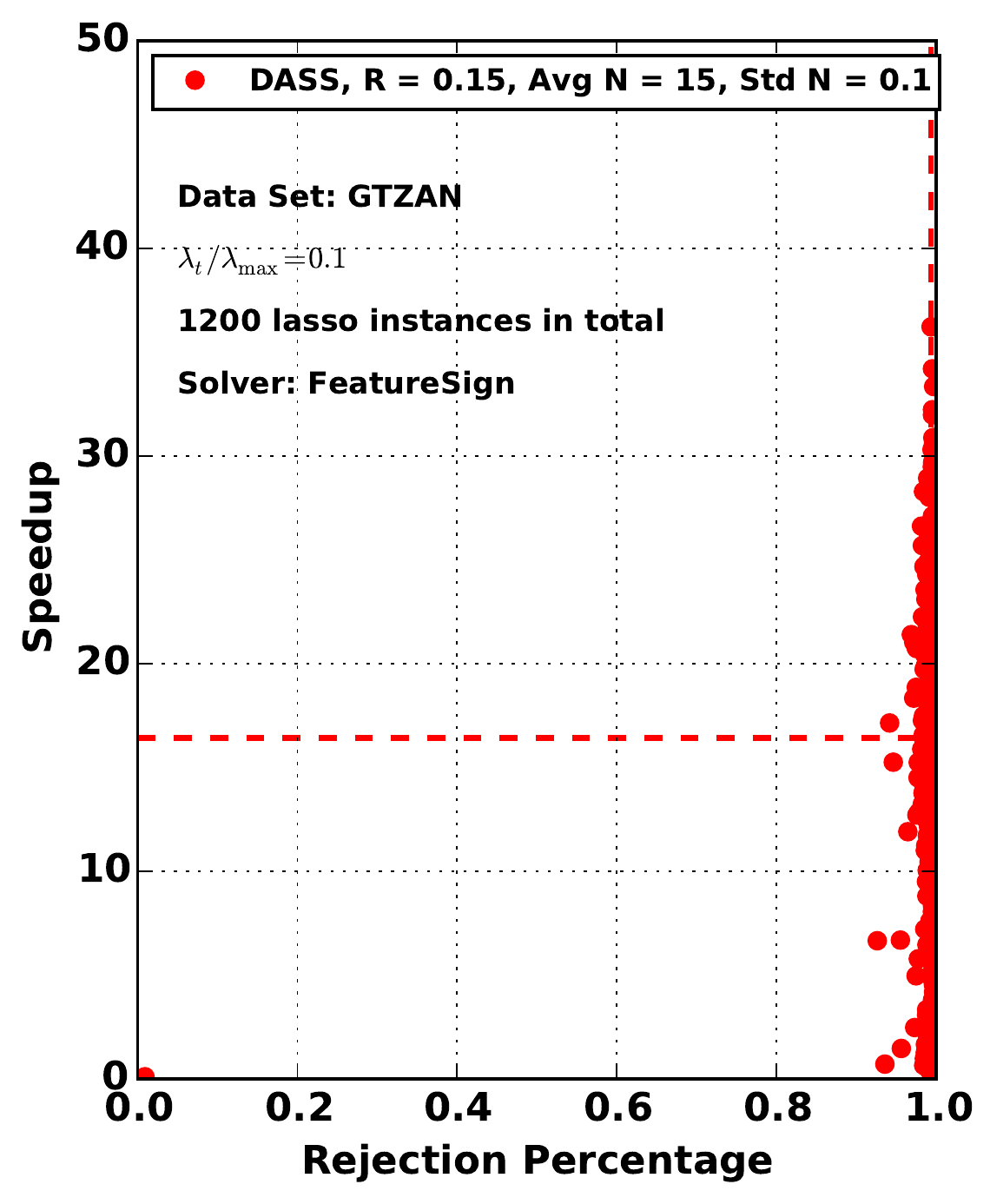}
\includegraphics[width=0.245\textwidth]{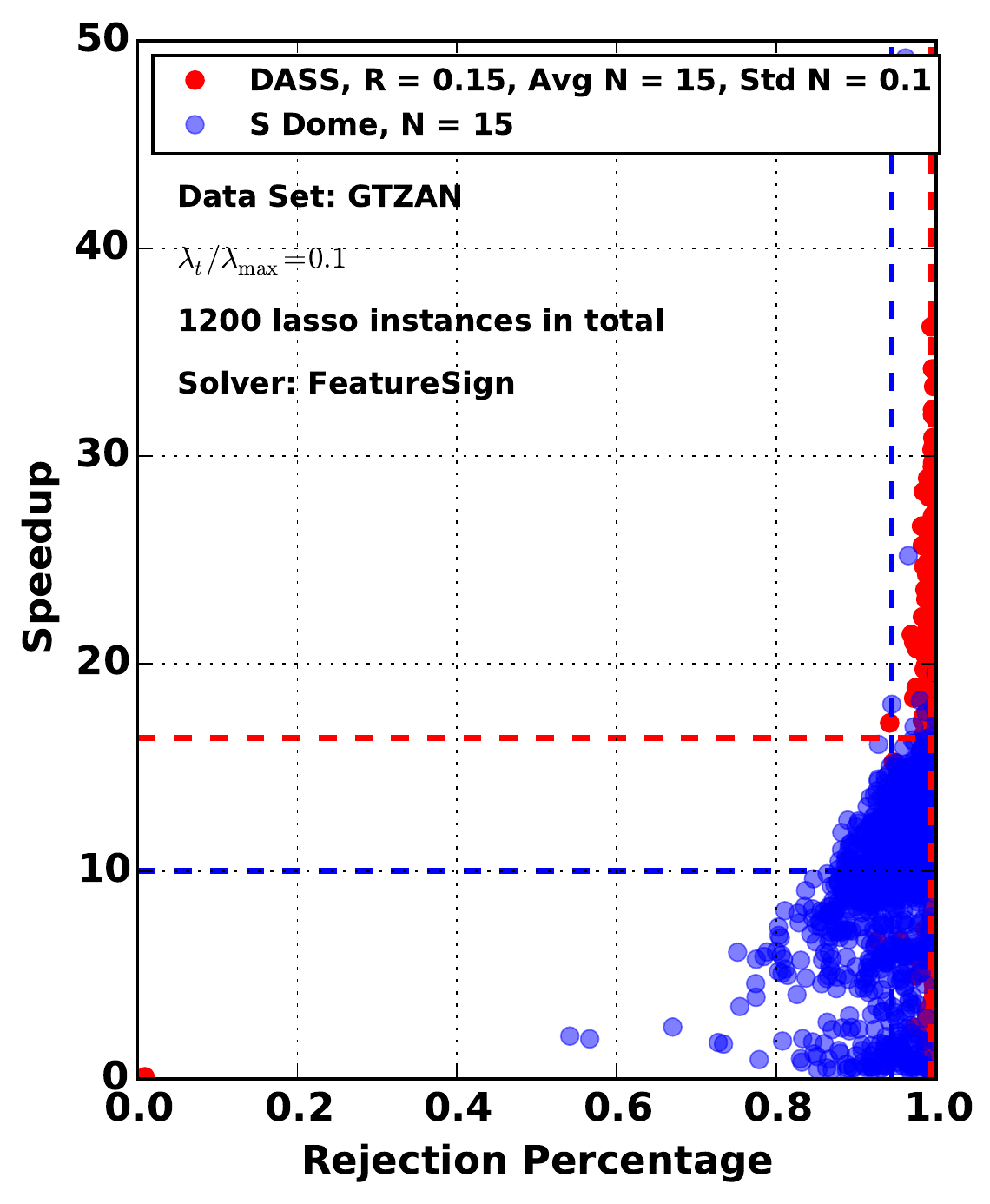}
\includegraphics[width=0.245\textwidth]{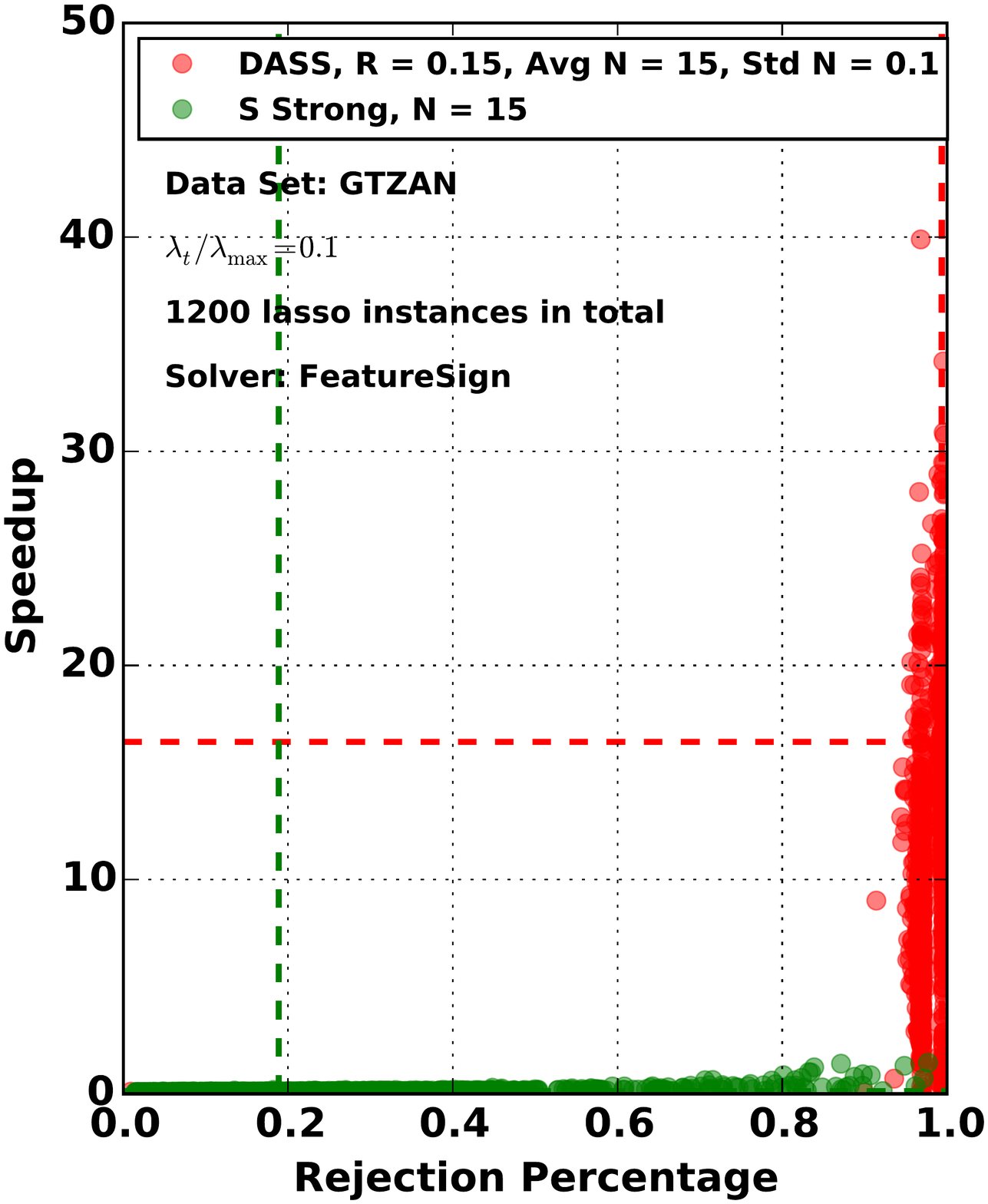}
\includegraphics[width=0.245\textwidth]{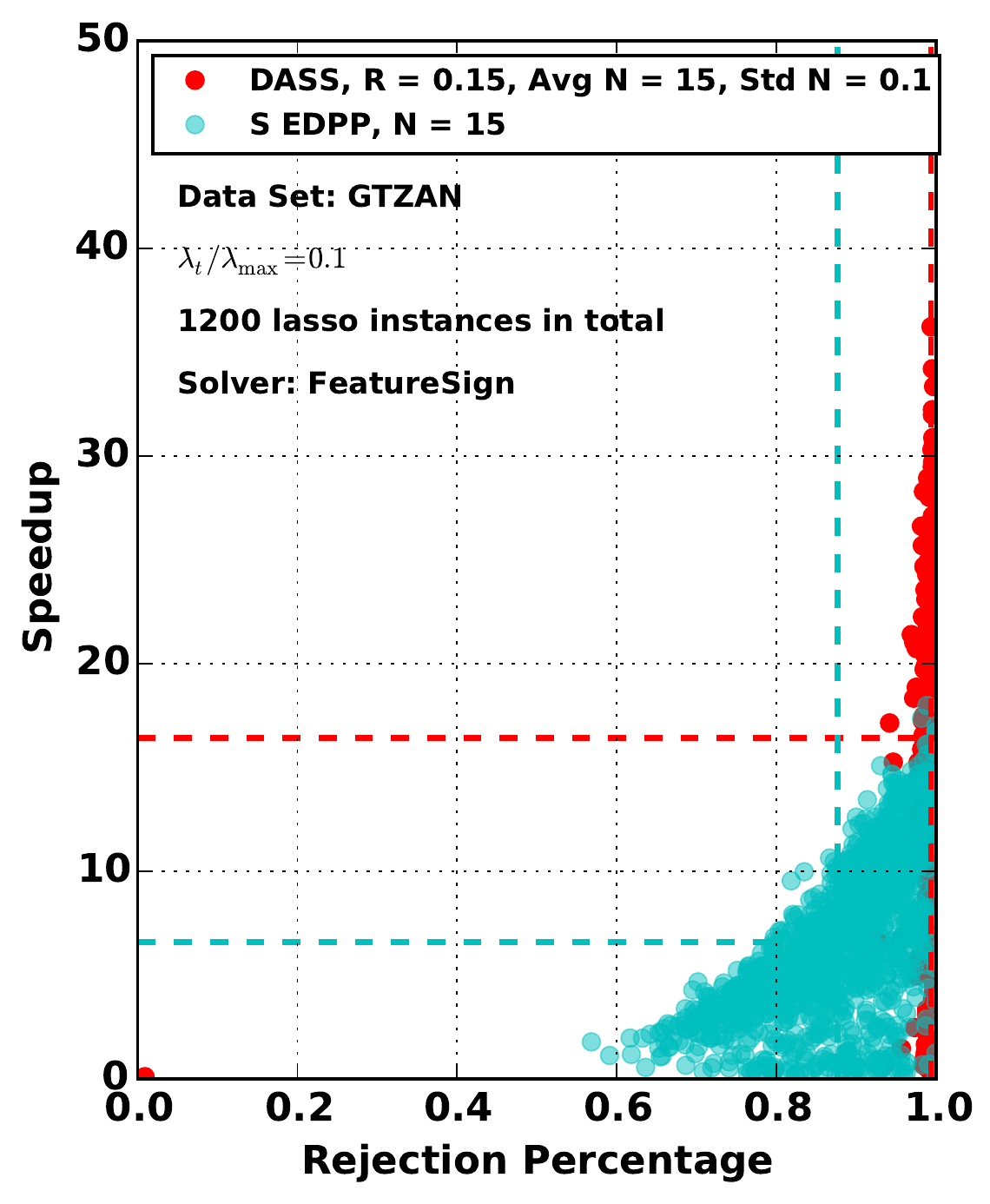}
\\
\includegraphics[width=0.49\textwidth]{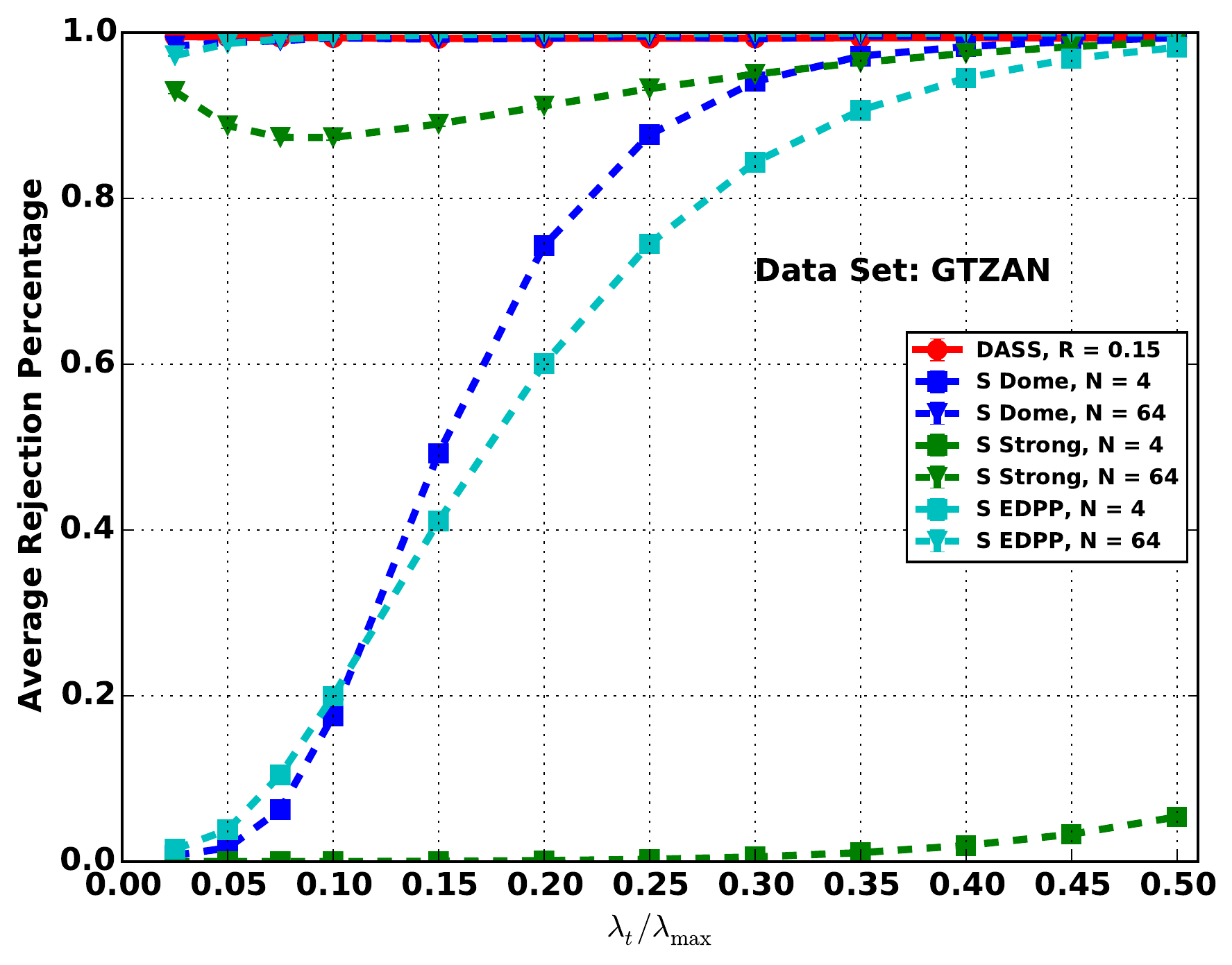}
\includegraphics[width=0.49\textwidth]{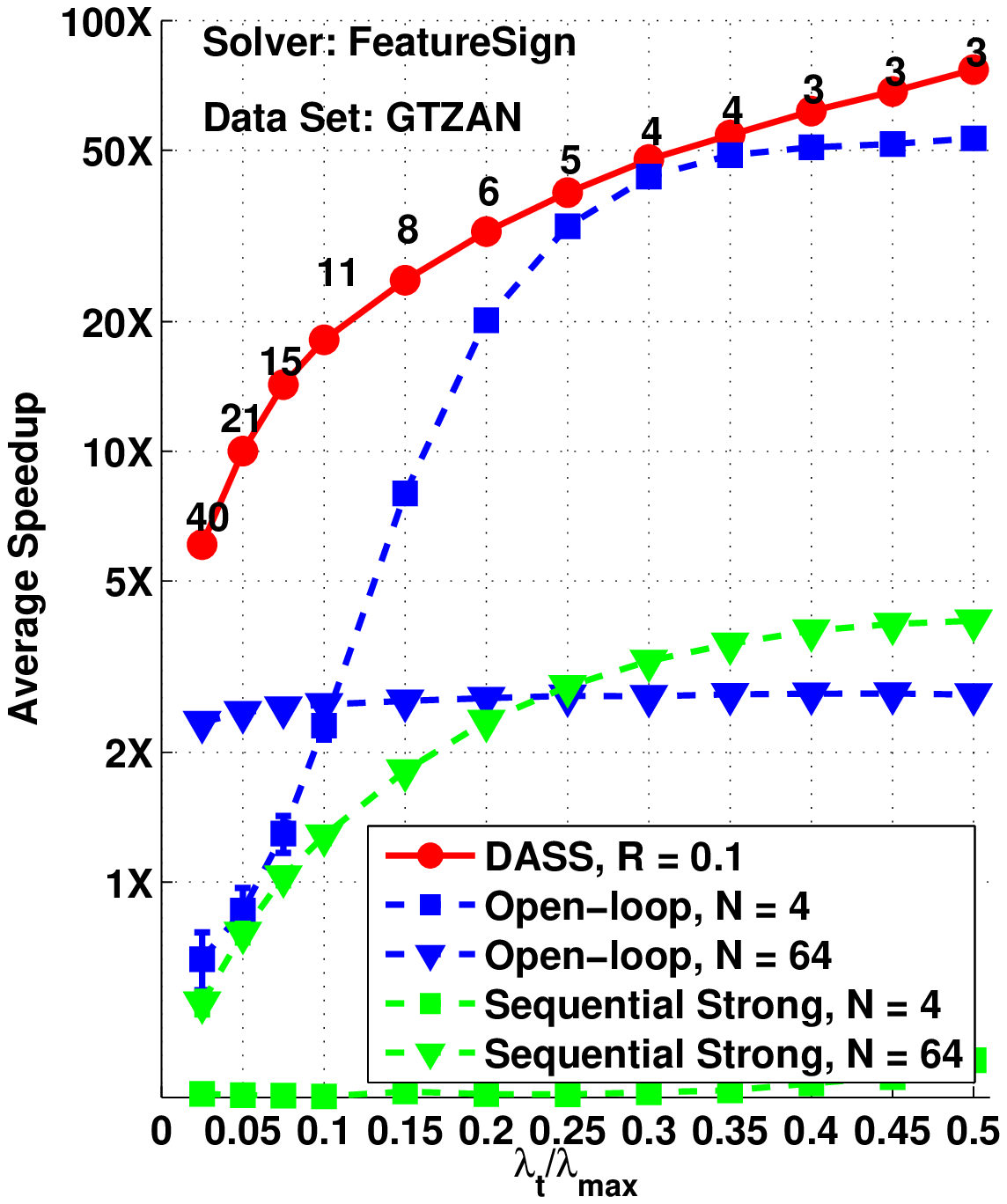}

\caption{
\small {\bf Comparison of DASS with sequential Dome, sequential Strong rule and sequential Enhanced DPP rule on GTZAN.}
(Top): Scatter plot of speedup versus rejection percentage for lasso instances ($D, \tv$) with $\lat/\lm=0.1$.
(Bottom): Average rejection and speedup for $\lat/\lm \in [0.025, 0.5]$. The average $N$ for DASS is shown on the speedup curve.
}
\label{fig:GTZAN}
\vspace{-0.7cm}
\end{figure}
We compare DASS with state-of-the-art open-loop sequential screening schemes with a fixed grid of geometrically spaced regularization parameters. These open-loop sequential screening schemes are sequential Dome \cite{YunWang2013b}, sequential Strong rule \cite{Tibshirani2010Strong} and sequential Enhanced DPP rule \cite{dpp2015}. Sequential Dome is the sequential version of the dome test proposed in \cite{Xiang2012Fast}. The authors of \cite{dpp2015} commented that \textit{``for the DOME  test, it is unclear whether its sequential version exists''} and did not include a comparison of their sequential Enhanced DPP rule with sequential Dome. Their claim is not true and we hereby include both algorithms for comparison. Among these four screening algorithms, only the sequential Strong rule in \cite{Tibshirani2010Strong} may falsely discard features.

\textbf{Performance Metrics}
\\ 1) Rejection Percentage: the ratio of the number of features rejected by screening rules to the number of \emph{all} features.
\\ 2) Speedup: the time to solve \eqref{eq:lasso} at $\lat$ without screening (\emph{one problem}) divided by the total time to screen \emph{plus} solve the reduced lasso problem along the sequence (\emph{N problems}).
\\Note the difference in our metrics definition compared with \cite{dpp2015}. In \cite{dpp2015} rejection percentage was defined as the ratio of the number of features discarded by screening rules
to the actual number of \emph{zero} features in the ground truth; and speedup was defined as the time to solve the sequence of lasso problems without screening (\emph{N problem}) divided by the total time to
screen plus solve the reduced lasso problem along the same sequence (\emph{N problem}). The difference in speedup definition mainly originates from our different motivation and application, as we discussed earlier in \S \ref{sec:intro}.

%
%

\textbf{Datasets}\\
(a) MNIST:
$28\times 28$ hand-written digit images (training set: $60,000$,  testing set: $10,000$) \cite{LeCun1998The-MNIST}. Each image is rearranged into a vector and scaled to unit norm. We sample $5,000$ training images (500 per digit) to form the dictionary. Test images $\tv$ are randomly selected from the testing set.\\
(b) GTZAN:
100 music clips (30 sec, sampled at 22,050 Hz) for each of ten genres of music \cite{Tzanetakis2002}.
Each clip is divided into 3-sec texture windows (TW) with
50\% overlap of adjacent TW. Each TW is then represented using a first order scattering vector of length 199 \cite{Mallat2011}.
We randomly select a dictionary of 12,000
scattering vectors and randomly select test vectors $\tv$
from the remaining 8,000. \\
(c) NYT: 300,000 New York Times articles represented as vectors with respect to a vocabulary of 102,660 words \cite{Frank+Asuncion:2010}. The $i$-th entry in vector $j$ gives the occurrence count of word $i$ in document $j$. After removing documents with very low word counts, 299,752 documents remain. 299,652 of these documents are randomly selected to form the dictionary and from the remaining 100 documents, 65 are selected as test vectors with $0.3 < \lm <0.9$.

For experiments on MNIST and GTZAN we randomly select 20 dictionaries and for each dictionary we use 60 randomly selected test vectors. Reported results are averaged across selected dictionaries and test vectors.
$\lat$ is set using the scaling invariant ratio $\fracn{\lat}{\lm} \in [0,1]$ with higher value yielding sparser solution. Speedup results using the Feature-sign lasso solver \cite{Lee2007Efficient} are shown but our experiment indicate consistent results for a variety of other solvers. For example, see the supplementary material for additional results using the FISTA \cite{AM2009} solver.

In an open-loop sequential screening scheme, once $\lat$ is given and $N$ is selected then the sequence $\{\lambda_k\}_{k=1}^N$ is fixed.
We claim a key attribute of DASS is that it adapts $\{\lambda_k\}_{k=1}^N$ to each individual problem: $(\Dict, \tv, \lat)$.
We test the effectiveness of this adaptation by solving 1200 problem instances with $\lat/\lm=0.1$ for both MNIST and GTZAN.
We then plot the scatter plot of speedup versus rejection percentage for these 1200 problem instances. See the top rows of Fig.~\ref{fig:MNIST} and Fig.~\ref{fig:GTZAN}.
%
For MNIST, the average $N$ used by DASS ($R=0.4$) is $21$ with standard error of $0.6$. DASS successfully pushes both the rejection percentage and speedup distribution to higher ranges compared with open-loop schemes using $N=21$.

DASS also exhibits robust average rejection percentage and speedup as $\lat/\lm$ varies (bottom rows Fig.~\ref{fig:MNIST} and \ref{fig:GTZAN}).
For MNIST, DASS consistently beats or rivals the open-loop schemes as $N$ ranges from $4$ to $64$. Open-loop schemes with $N=64$ are close to DASS in rejection, but its average speedup is only around $2$ or less for most values of $\lat/\lm$. The DASS speedup curve gives a nice outer envelope across the set of open-loop speedup curves.
The average $N$ used by DASS corresponding to each value of $\lat/\lm$ is given on top of the DASS speedup curve (red).
As $\lat/\lm$ decreases, DASS increases $N$ on a case-by-case basis, and this is done without user intervention. Using a properly cross-validated $R$, DASS is able to automatically design an efficient sequence $\{\lambda_k\}_{k=1}^N$ for each individual $(D, \tv, \lat)$.
Similar results can been seen for GTZAN. Since overhead accumulates as $N$ increases, DASS, and open-loop sequential screening schemes, both exhibit less speedup as $\lat/\lm$ decreases.
Nevertheless, DASS pushes efficient and robust screening down to $\lat/\lambda_{\max} = 0.025$, which includes a range useful in many classification applications.

The DASS scheme is easy to implement, and its one parameter
can be selected by cross-validation using one reasonable
value of $\lat/\lm$. The scheme then gives strong screening performance across a wide range of values of $\lat/\lm$. A smaller value of the user-specified parameter $R$, yields a tighter bound on $\diam(\RR^0_k)$, and hence stronger rejection, but also a larger value of  $N$.
Hence speedup will hit a sweet spot as $R$ decreases. Empirical investigation on the effect of $R$ is given in Fig.~\ref{fig:R1}.
\begin{figure}[h]
\centering
\includegraphics[width=0.315\textwidth]{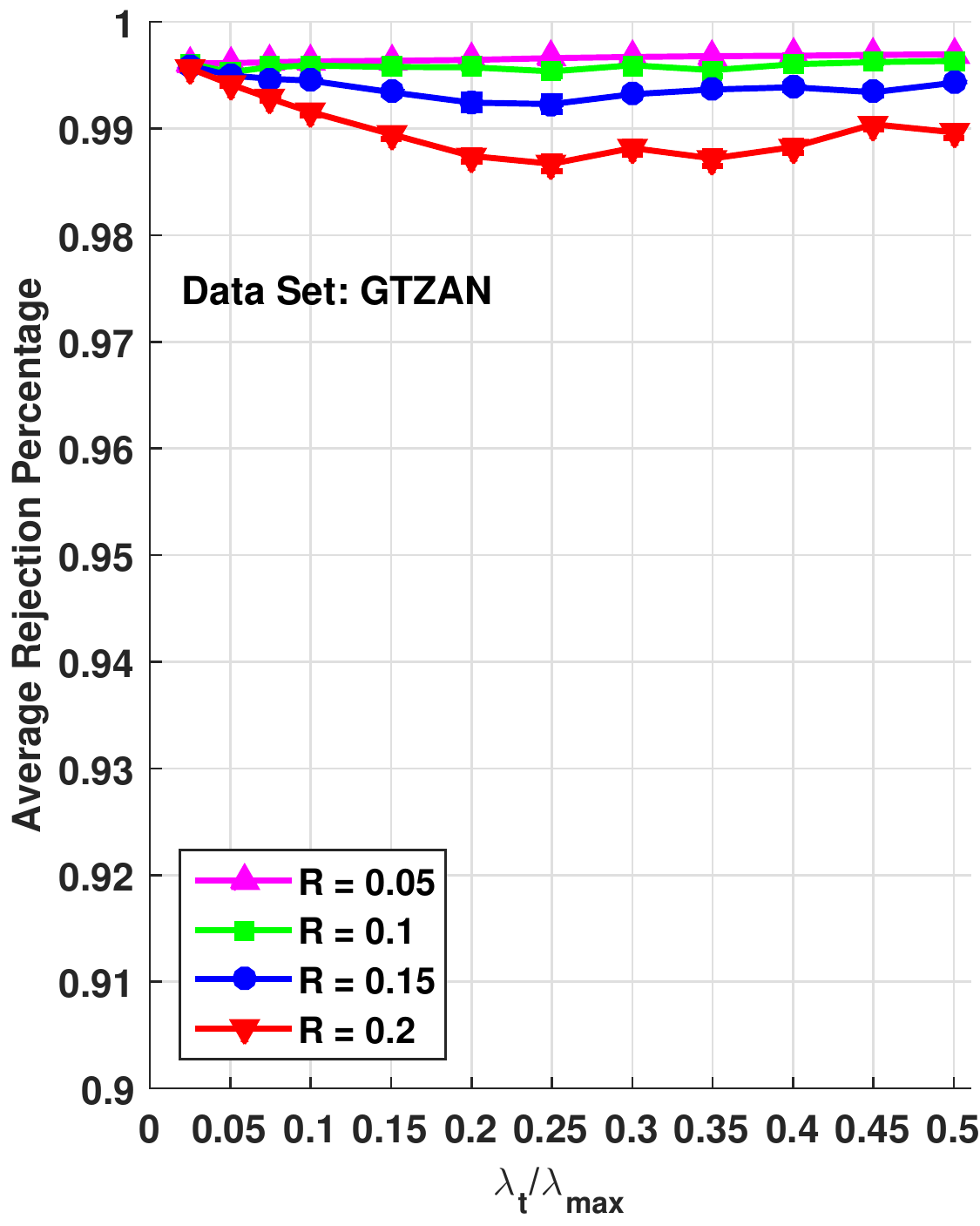}
\includegraphics[width=0.315\textwidth]{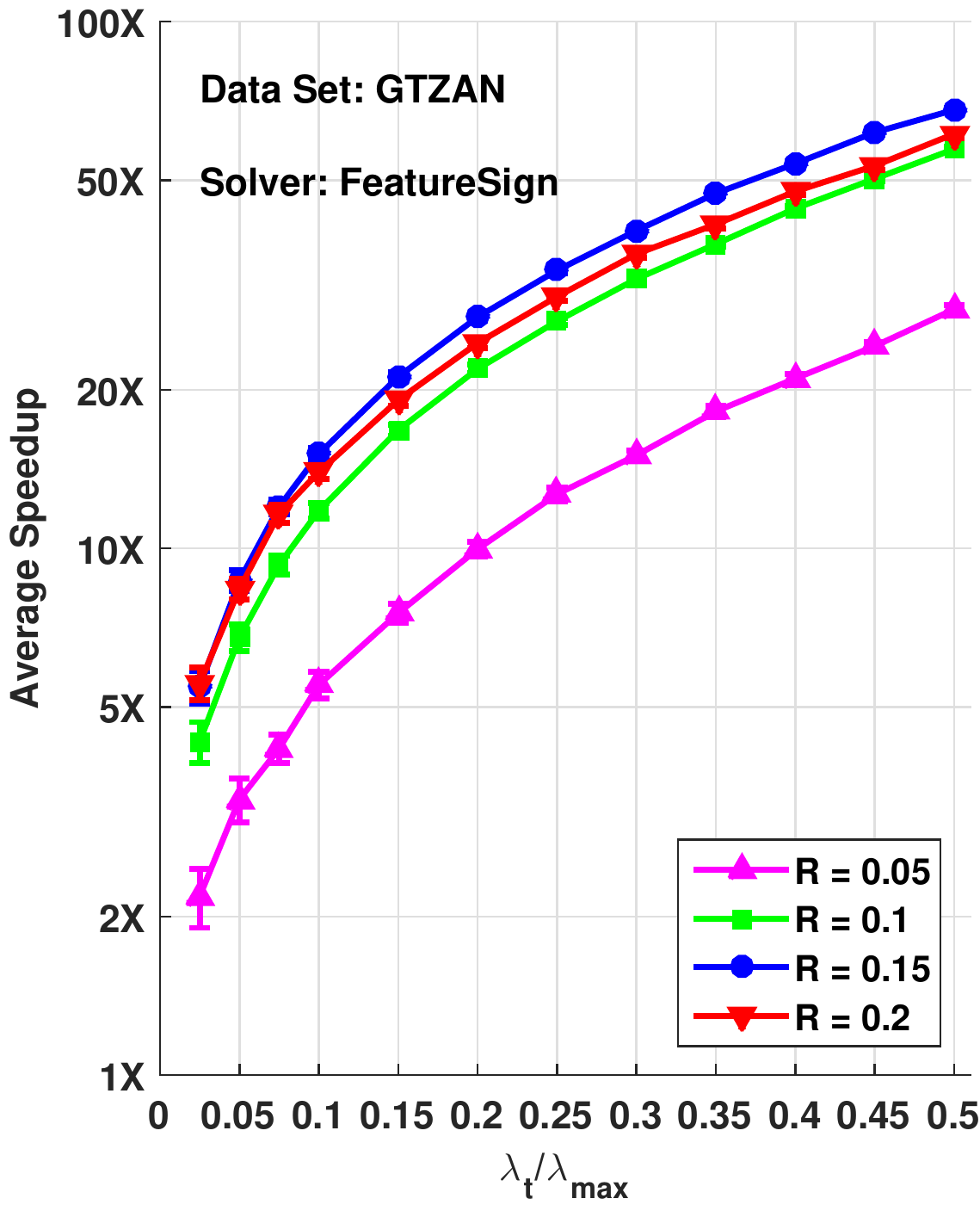}
\includegraphics[width=0.315\textwidth]{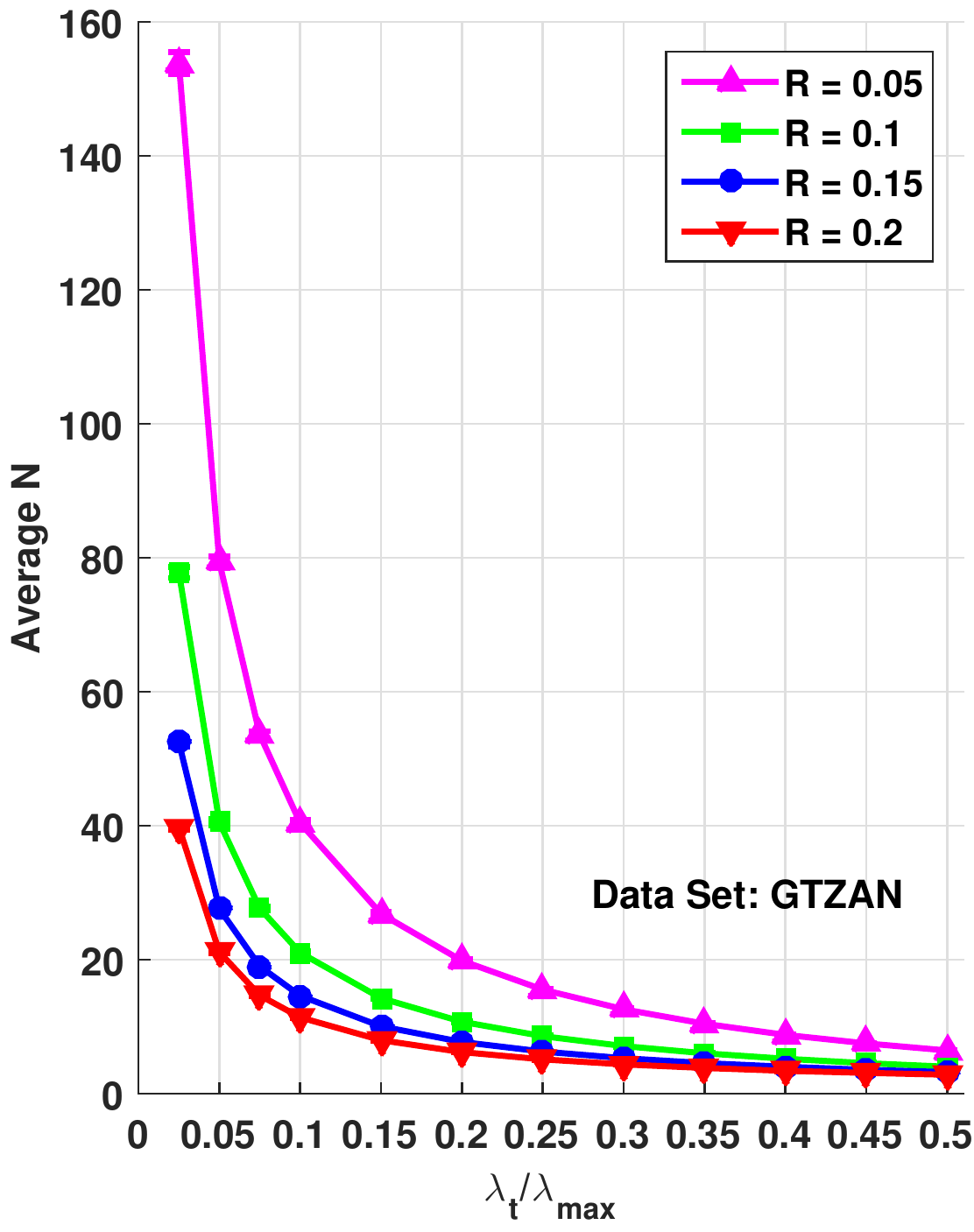}

\caption{
\small {\bf Effects of the choice of $R$ on DASS performance.}
(Left): Average rejection percentage.
(Middle): Average speedup.
(Right): The average value of $N$ over the $1200$ trials.
}
\label{fig:R1}
\end{figure}

To examine the effect of ``noisy'' lasso solutions,
we add noise $\vn_k \sim N(0,\alpha^2 I)$ to our computed $\wvo_k$, $k=1\mc N$, for various values of noise-to-signal ratio (the ratio of the power of $\vn_k$ to that
of $\wvo_k$). We then hard threshold $\hat{\wv}_k+\vn_k$ to yield a ``corrupted'' sparse solution. The impact on rejection fraction, speedup, and classification accuracy is shown for MNIST in Fig.~\ref{fig:error}. Rejection remains around $0.9$ until the noise-to-signal ratio goes above $10^{-3}$.
Speedup drops from around $5X$ to $1X$ at $10^{-6}$.
So DASS remains effective for noise-to-signal ratio below $10^{-6}$.
Within this bound, there is a small penalty for the added corruption, but
DASS is robust to the errors. In addition,
when $\tilde{\wv}_N$ is used for Sparse Representation Classification \cite{Wright2009Robust}, classification accuracy is also robust to the error in $\wvo_N$.
\begin{figure}[t!]
\centering
\includegraphics[width=0.3\textwidth]{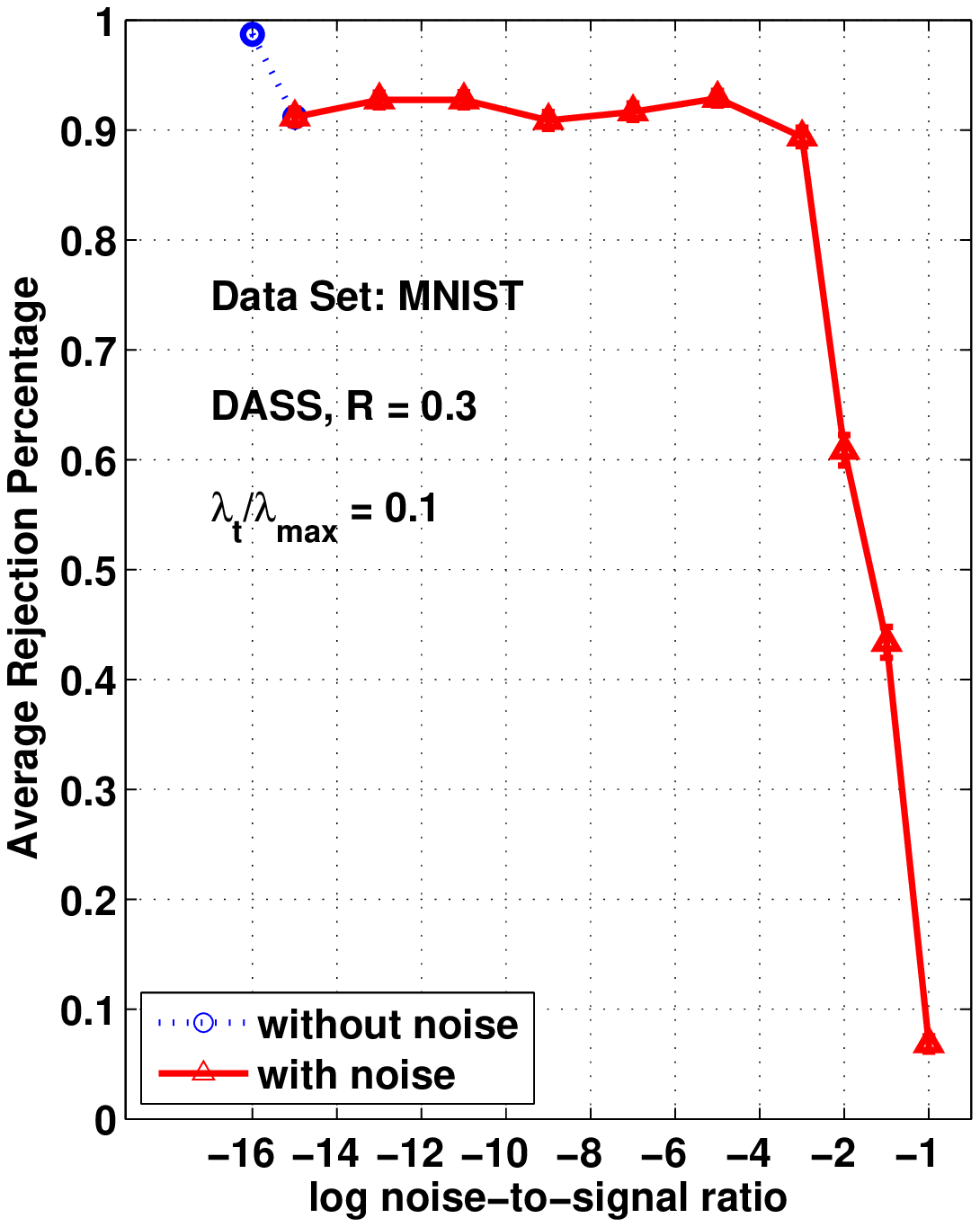}
\hfil
\includegraphics[width=0.29\textwidth]{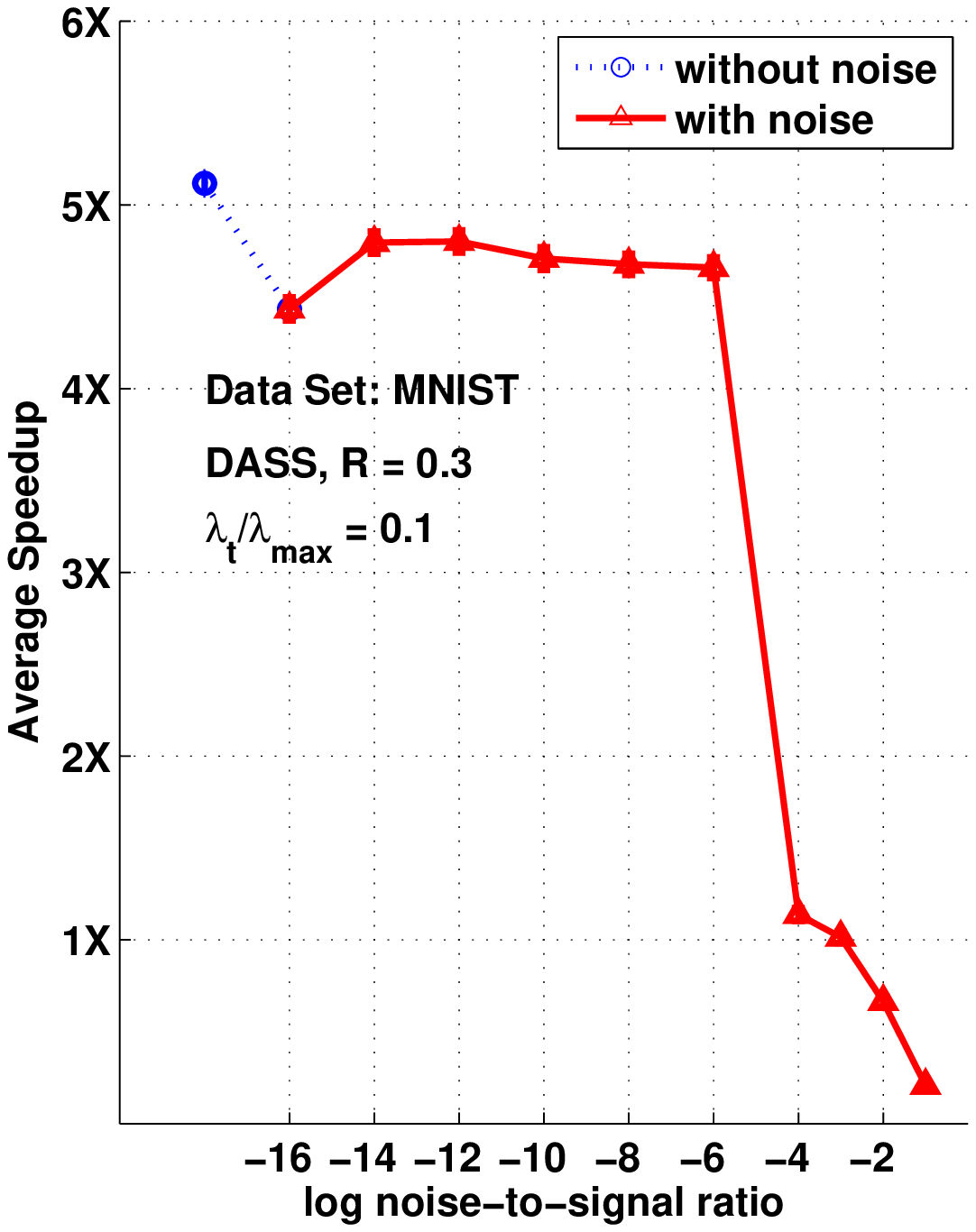}
\hfil
\includegraphics[width=0.3\textwidth]{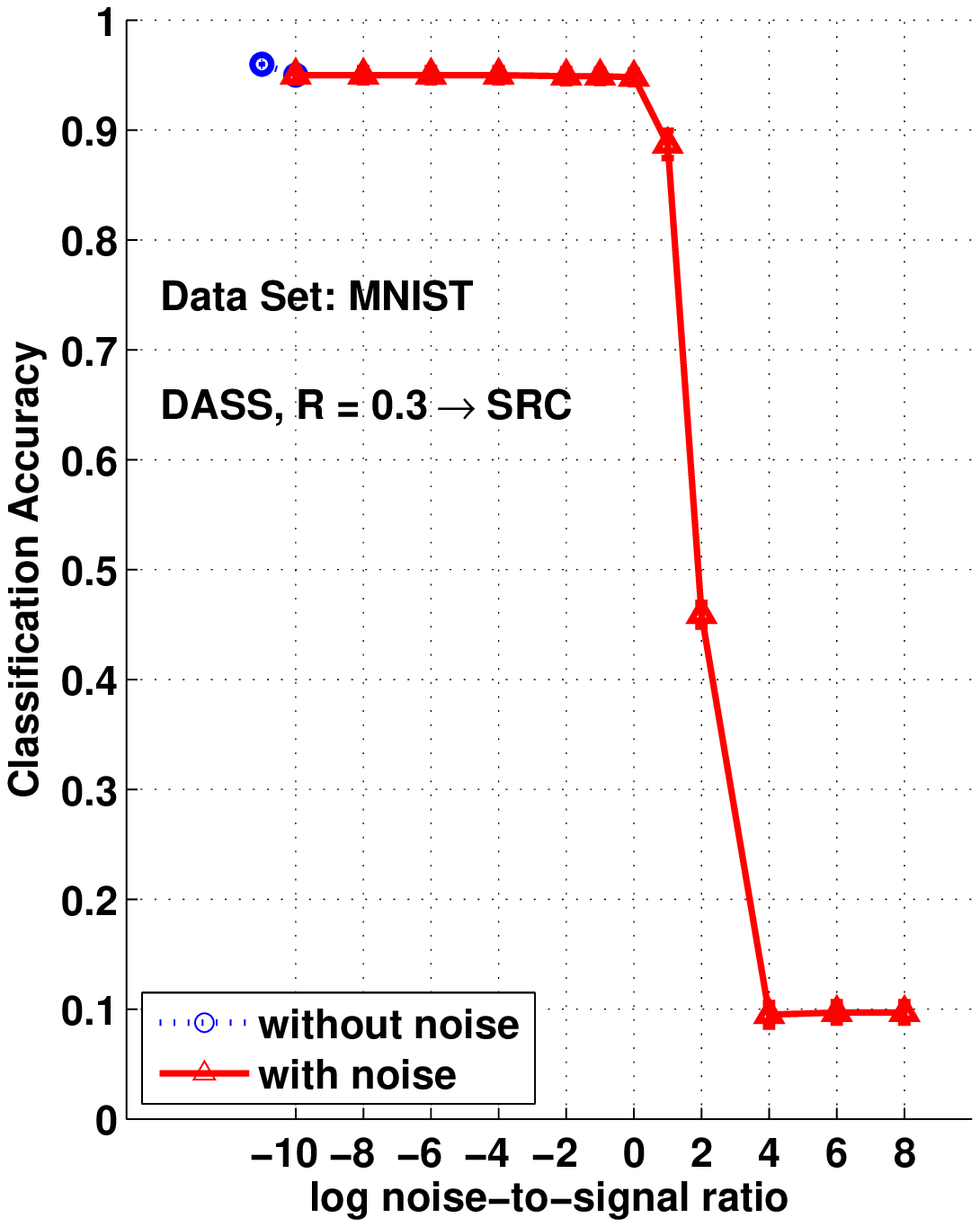}
\caption{\small
{\bf Effects of adding noise to $\wvo_k$ in DASS.}
}
\label{fig:error}
\end{figure}

Finally, we use the NYT dataset to explore problems with many high dimensional features.
From the test vector pool of $100$ documents, we select $65$ documents, subject to $0.3 <\lm <0.9$. $\lat/\lm$ is set at $0.1$.
Since the dictionary cannot fit into our computer memory,
we run screening by loading small segments of the dictionary into memory at a time. We test DASS with $R=0.3$.
The average resulting values of $N$ from $65$ problem instances is $45$ with standard error being $1.8$.
For comparison, we run sequential Dome, sequential Strong rule and sequential Enhanced DPP rule with $N=65$.

We allocate the same amount of system resources to each algorithm and each problem instance. Sequential Dome and sequential EDPP are not able to complete all $65$ problem instances because for some instances they do not reject enough features in the process and run into out of memory error. We thus include the completion rate metric in Fig.~\ref{fig:NYT} defined as the percentage of problem instances that are completed given the same system resources. DASS and sequential Strong rule take the lead in this metric and are able to complete all problem instances. For problem instances that are completed by each algorithm, we provide a scatter plot of total running time versus number of remaining features in Fig.~\ref{fig:NYT}.
No speedup information is shown since we cannot solve these problems without screening. The results show that DASS is on a par with sequential Strong rule and they significantly outperform sequential Dome and sequential EDPP rules: sequential Strong rule is able to reject more but DASS is able to finish in less runtime. Note however that the Strong rule may commit false rejections.

\begin{figure}[t!]
\centering
\includegraphics[width=0.6\textwidth]{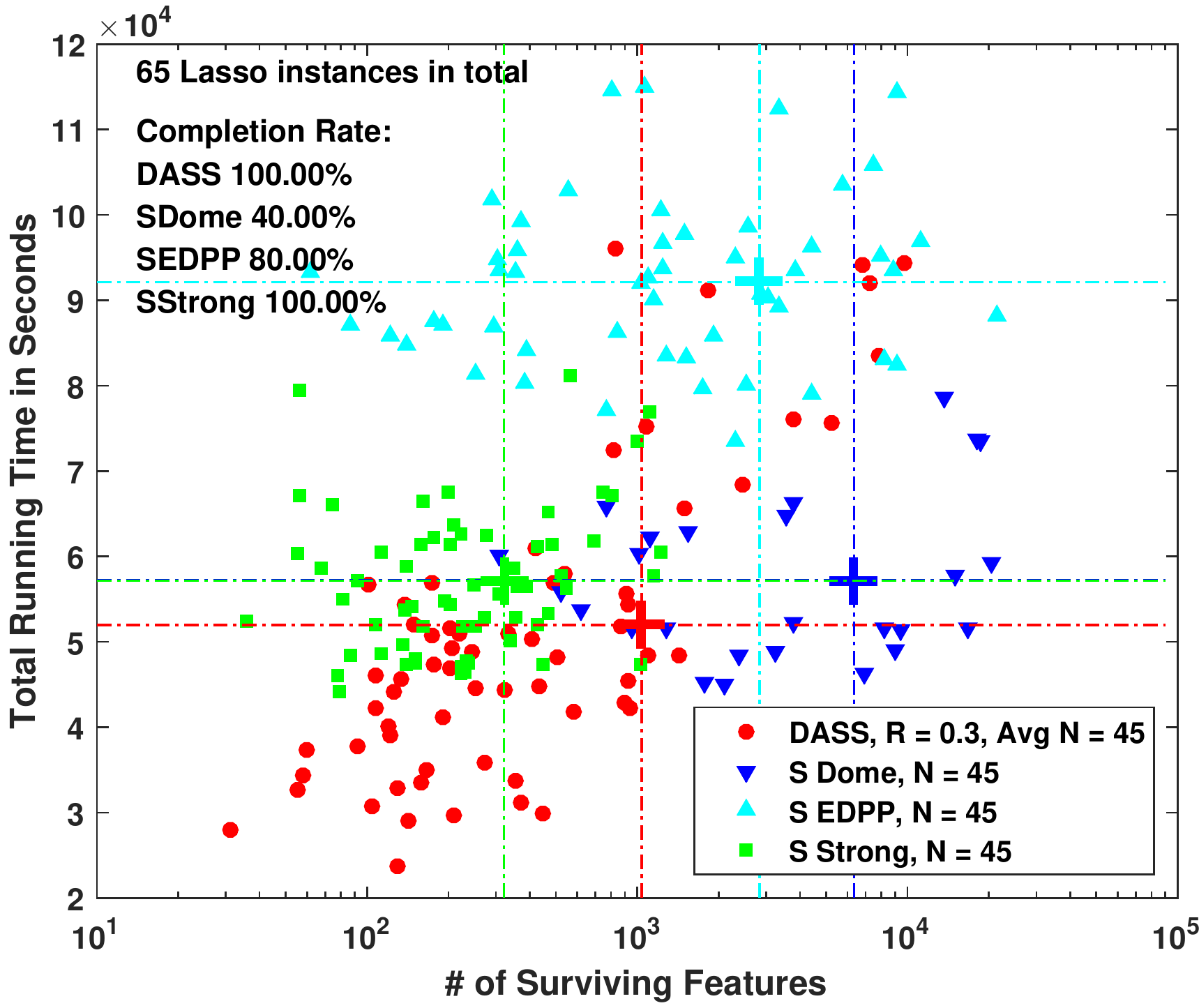}
\caption{\small
{\bf Comparison of DASS with sequential Dome, sequential Strong rule and sequential Enhanced DPP rule on NYT. } Both DASS and three open-loop sequential screening schemes are applied to $65$ lasso instances with $0.3<\lm<0.9$ and $\lat/\lm = 0.1$. The colored $+$ sign marks the average performance for each algorithm.
}
\vspace{-0.6cm}
\label{fig:NYT}
\end{figure}

\section{Conclusion}\label{sec:conclude}
We have shown analytically and empirically that
feedback controlled sequential screening yields improved
performance. Given a cross-validated target $\lat$, the DASS rule tailors an individualized sequence $\{\lambda_k\}_{k=1}^N$ on the fly which is optimized for each particular problem instance $(\Dict, \tv, \lat)$. This has demonstrated significantly greater robustness with respect to the inherent variability in the lasso problem.

Our feedback idea extends to the DPP bound used in \cite{dpp2015}.
The selection rule $\lambda_k =(1/2R+\lambda^{-1}_{k-1})^{-1}$ yields a feedback controlled version of sequential DPP \cite{dpp2015}. This uses a looser bound than $\RR_k^0$ (see Fig.~\ref{fig:boundonth}).
%
Our results also suggest that it may be possible to use
feedback to control sequential screening in more targeted ways.
For example, by varying the parameter  $R$ in response to previously
solved instances it may be possible to regulate the size of the dictionary after screening.
This would be useful in situations of a large number of features and limited  memory.

{
\small
\bibliographystyle{plain}
\bibliography{Lassobib,books,Screening2014,musicbib}

\begin{thebibliography}{10}

\bibitem{Mallat2011}
J.~And{\'e}n and S.~Mallat.
\newblock Multiscale scattering for audio classification.
\newblock {\em Proc. of the ISMIR Conference}, 2011.

\bibitem{AndenMallat2013}
J.~And{\'e}n and S.~Mallat.
\newblock Deep scattering spectrum.
\newblock {\em IEEE Trans. on Signal Processing}, 2013.

\bibitem{Beck2009}
A.~Beck and M.~Teboulle.
\newblock A fast iterative shrinkage-thresholding algorithm for linear inverse
  problems.
\newblock {\em SIAM J. Imaging Sciences}, 2(1):183--202, 2009.

\bibitem{AM2009}
Amir Beck and Marc Teboulle.
\newblock A fast iterative shrinkage-thresholding algorithm for linear inverse
  problems.
\newblock {\em SIAM J. IMAGING SCIENCES}, 2(1):183--202, 2009.

\bibitem{Boyd2004Convex}
S.~P. Boyd and L.~Vandenberghe.
\newblock {\em {Convex optimization}}.
\newblock Cambridge Univ Pr, 2004.

\bibitem{RCV1}
D.~Cai and X.~He.
\newblock Manifold adaptive experimental design for text categorization.
\newblock {\em IEEE Transactions on Knowledge and Data Engineering}, April
  2011.

\bibitem{XuChen2013}
Xu~Chen and P.~J. Ramadge.
\newblock Music genre classification using multiscale scattering and sparse
  representations.
\newblock In {\em Proc. Annual Conference on Information Sciences and Systems},
  2013.

\bibitem{ellpscr2012}
L.~Dai and K.~Pelckmans.
\newblock An ellipsoidal based, two-stage screening test for bpdn.
\newblock In {\em 20th European Signal Processing Conference}, Aug. 2012.

\bibitem{Ghaoui2012}
L.~El~Ghaoui, V.~Viallon, and T.~Rabbani.
\newblock Safe feature elimination in sparse supervised learning.
\newblock {\em Pacific Journal of Optimization}, 8(4):667--698, 2012.

\bibitem{Elad2010Sparse}
M.~Elad.
\newblock {\em {Sparse and Redundant Representations: From Theory to
  Applications in Signal and Image Processing}}.
\newblock Springer, 2010.

\bibitem{GramfortICML15}
O.~Fercoq, A.~Gramfort, and J.~Salmon.
\newblock Mind the duality gap: safer rules for the lasso.
\newblock In {\em International Conference on Machine Learning}, 2015.

\bibitem{Frank+Asuncion:2010}
A.~Frank and A.~Asuncion.
\newblock {UCI Machine Learning Repository, University of California, Irvine,
  School of Information and Computer Sciences}, 2010.

\bibitem{Georghiades2002From}
A.~S. Georghiades, P.~N. Belhumeur, and D.~J. Kriegman.
\newblock {From few to many: Illumination cone models for face recognition
  under variable lighting and pose}.
\newblock {\em IEEE Trans. on Pattern Analysis and Machine Intelligence},
  23(6):643--660, 2002.

\bibitem{DeepSparseCoding}
Y.~He, K.~Kavukcuoglu, Y.~Wang, A.~Szlam, and Y.~Qi.
\newblock Unsupervised feature learning by deep sparse coding.
\newblock In {\em SIAM International Conference on Data Mining}, pages
  902--910. SIAM, 2014.

\bibitem{HU2001}
J.-B. Hiriart-Urruty and C.~Lemarechal.
\newblock {\em Fundamentals of Convex Analysis}.
\newblock Springer, 2001.

\bibitem{LeCun1998The-MNIST}
Y.~LeCun and C.~Cortes.
\newblock {The MNIST database of handwritten digits}, 1998.

\bibitem{Lee2007Efficient}
H.~Lee, A.~Battle, R.~Raina, and A.Y. Ng.
\newblock {Efficient sparse coding algorithms}.
\newblock In {\em Advances in Neural Information Processing Systems},
  volume~19, page 801, 2007.

\bibitem{Lee2005Acquiring}
K.~C. Lee, J.~Ho, and D.~J. Kriegman.
\newblock {Acquiring linear subspaces for face recognition under variable
  lighting}.
\newblock {\em IEEE Trans. on Pattern Analysis and Machine Intelligence}, pages
  684--698, 2005.

\bibitem{GramfortNIPS15}
E.~Ndiaye, O.~Fercoq, A.~Gramfort, and J.~Salmon.
\newblock Gap safe screening rules for sparse multi-task and multi-class
  models.
\newblock In {\em Advances in Neural Information Processing Systems}, 2015.

\bibitem{GramfortNIPS16}
E.~Ndiaye, O.~Fercoq, A.~Gramfort, and J.~Salmon.
\newblock Gap safe screening rules for sparse-group lasso.
\newblock In {\em Advances in Neural Information Processing Systems}, 2016.

\bibitem{Nene1996Columbia}
S.~A. Nene, S.~K. Nayar, and H.~Murase.
\newblock {Columbia object image library (coil-100)}.
\newblock {\em Techn. Rep. No. CUCS-006-96, dept. Comp. Science, Columbia
  University}, 1996.

\bibitem{Panagakis2009}
Y.~Panagakis, C.~Kotropoulos, and G.R. Arce.
\newblock Music genre classification using locality preserving non-negative
  tensor factorization and sparse representations.
\newblock {\em Proc. of the ISMIR Conference}, 2009.

\bibitem{Prasad2011Emerging}
S.~Prasad, P.~Melville, A.~Banerjee, and V.~Sindhwani.
\newblock {Emerging topic detection using dictionary learning}.
\newblock In {\em ACM Conference on Information and Knowledge Management},
  2011.

\bibitem{RobTib1996}
R.~Tibshirani.
\newblock Regression shrinkage and selection via the lasso.
\newblock {\em J. Royal. Statist. Soc B.}, 58(1):267--288, 1996.

\bibitem{Tibshirani2010Strong}
R.~Tibshirani, J.~Bien, J.~Friedman, T.~Hastie, N.~Simon, J.~Taylor, and R.~J.
  Tibshirani.
\newblock {Strong rules for discarding predictors in lasso-type problems}.
\newblock {\em Arxiv preprint arXiv:1011.2234}, 2010.

\bibitem{TroppGilbert2007}
J.A. Tropp and A.C. Gilbert.
\newblock Signal recovery from random measurements via orthogonal matching
  pursuit.
\newblock {\em Information Theory, IEEE Transactions on}, 53(12):4655--4666,
  Dec 2007.

\bibitem{Tzanetakis2002}
G.~Tzanetakis and P.~Cook.
\newblock Musical genre classification of audio signals.
\newblock {\em IEEE Trans. Speech and Audio Processing}, 10(5), 2002.

\bibitem{Wagner2011Towards}
A.~Wagner, J.~Wright, A.~Ganesh, Z.~Zhou, H.~Mobahi, and Y.~Ma.
\newblock {Towards a practical face recognition system: robust alignment and
  illumination by sparse representation}.
\newblock {\em IEEE Transactions on Pattern Analysis and Machine Intelligence},
  2011.

\bibitem{dpp2015}
J.~Wang, P.~Wonka, and J.~Ye.
\newblock Lasso screening rules via dual polytope projection.
\newblock {\em Journal of Machine Learning Research}, to appear.

\bibitem{WangThesis2015}
Y.~Wang.
\newblock {\em Feature Screening for the Lasso}.
\newblock PhD thesis, Princeton University, Sept. 2015.

\bibitem{SymmetryOptimization}
Y.~Wang and P.~J. Ramadge.
\newblock The symmetry of a simple optimization problem in lasso screening.
\newblock {\em arXiv:1608.06014 [cs.LG]}, 2016.

\bibitem{YunWang2013b}
Y.~Wang, Z.~J. Xiang, and P.~J. Ramadge.
\newblock Lasso screening with a small regularization parameter.
\newblock In {\em IEEE Int. Conf. on Acoustics, Speech and Signal Processing},
  Jun. 2013.

\bibitem{YunWang2013a}
Y.~Wang, Z.J. Xiang, and P.J. Ramadge.
\newblock Tradeoffs in improved screening of lasso problems.
\newblock In {\em IEEE Int. Conf. on Acoustics, Speech and Signal Processing},
  Jun. 2013.

\bibitem{SRCvSLS2013}
Yun Wang, Xu~Chen, and P.~J. Ramadge.
\newblock Sparse representation classification via sequential lasso screening.
\newblock In {\em Global Conference on Signal and Information Processing
  (GlobalSIP), 2013 IEEE}, pages 1001--1004, Dec. 2013.

\bibitem{Wright2010Sparse}
J.~Wright, Y.~Ma, J.~Mairal, G.~Sapiro, T.~Huang, and S.~Yan.
\newblock {Sparse representation for computer vision and pattern recognition}.
\newblock {\em Proceedings of the IEEE}, 98(6):1031--1044, 2010.

\bibitem{Wright2009Robust}
J.~Wright, A.~Y. Yang, A.~Ganesh, S.~S. Sastry, and Y.~Ma.
\newblock {Robust face recognition via sparse representation}.
\newblock {\em IEEE Trans, on Pattern Analysis and Machine Intelligence},
  31(2):210--227, 2009.

\bibitem{Xiang2012Fast}
Z.~J. Xiang and P.~J. Ramadge.
\newblock {Fast lasso screening tests based on correlations}.
\newblock In {\em IEEE Int. Conf. on Acoustics, Speech and Signal Processing},
  2012.

\bibitem{LSS}
Z.~J. Xiang, Yun Wang, and P.~J. Ramadge.
\newblock Screening tests for lasso problems.
\newblock {\em IEEE Trans. Pattern Analysis and Machine Intelligence}, DOI:
  10.1109/TPAMI.2016.2568185., 12 May 2016.

\bibitem{Xiang2011Learning_b}
Z.~J. Xiang, H.~Xu, and P.~J. Ramadge.
\newblock {Learning sparse representations of high dimensional data on large
  scale dictionaries}.
\newblock In {\em Advances in Neural Information Processing Systems}, 2011.

\bibitem{PMXu2013}
Pingmei Xu and P.~J. Ramadge.
\newblock Three structural results on the lasso.
\newblock In {\em IEEE Int. Conf. on Acoustics, Speech and Signal Processing},
  2013.

\end{thebibliography}
}

\newpage   
\section*{Supplementary Material}

\subsection*{Proof of Lemma \ref{lem:diamR}}
\begin{proof}
We want to find the solution to
$\min_{\vth_1,\vth_2} -\nhalf \|\vth_1-\vth_2\|_2^2$ subject to:
$\|\vth_k-\vq\|_2^2\leq r^2$ and $\vn^T\vth_k\leq c$, $k=1,2$.
Form the Lagrangian
$$
\textstyle
L=-\nhalf \|\vth_1-\vth_2\|_2^2
+ \sum_{k=1}^2 \nhalf \lambda_k(\|\vth_k-\vq\|_2^2-r^2)
+ \sum_{k=1}^2 \mu_k(\vn^T\vth_k-c)
$$
The objective function and inequality constraint functions are all convex $C^1$ functions, and by the assumption of Lemma \ref{lem:diamR} there exist strictly feasible points. Hence the KKT conditions are necessary and sufficient for optimality \cite[P. 244]{Boyd2004Convex}. Thus for $\vth_1,\vth_2$ to be a solution, it is sufficient that these points satisfy the four constraints,
the complementary slackness conditions
$\lambda_k(\|\vth_k-\vq\|_2^2-r^2)=0$ and
$\mu_k(\vn^T\vth_k-c) =0$, $k=1, 2$, and that
$(dL/d\vth_1)(\vth_1)=(dL/d\vth_2)(\vth_2)=0$.

We now construct a candidate solution.
Suppose $\vq$ is not contained in $R$.
Let $\vz$ denote the orthogonal projection of
$\vq$ onto the hyperplane $\vn^T\vth=c$.
Since $\vq\notin R$, $\vn^T\vq>c$
and $\vz=\vq-\fr r \vn$ for some $0<\fr\leq 1$.
Since $\vz$ lies on the plane, we must have
$c=\vn^T\vz=\vn^T\vq-\fr r$.
Hence
$$
\fr= (\vn^T\vq-c)/r.
$$
Pick any unit norm vector $\vu$ with $\vn^T\vu=0$.
Then set
\begin{equation}\label{eq:defvtk}
\vth_1=\vz+r\sqrt{1-\fr^2} \vu
\qquad
\vth_2=\vz-r\sqrt{1-\fr^2} \vu.
\end{equation}
For these points
$\|\vth_k-\vq\|_2^2
= \|-r\fr\vn \pm r\sqrt{1-\fr^2} \vu\|_2^2
=  r^2\fr^2 + r^2(1-\fr^2) = r^2
$
and
$
\vn^T \vth_k = \vn^T(\vz \pm \fr r \vu) =c
$.
So both points satisfy the inequality constraints
with equality and hence also satisfy the complementary
slackness conditions.
The final two KKT conditions require
\begin{align}
D_{\vth_1} L (\vth_1)
&= -(\vth_1-\vth_2) +\lambda_1(\vth_1-\vq) +\mu_1\vn =0
\label{eq:D1L}\\
D_{\vth_2} L (\vth_2)
&= -(\vth_2-\vth_1) +\lambda_2(\vth_2-\vq) +\mu_2\vn =0
\label{eq:D2L}
\end{align}
Consider equation \eqref{eq:D1L}.
Substituting $\vth_1, \vth_2$ from \eqref{eq:defvtk}
yields
$$
r\sqrt{1-\fr^2}(\lambda_1-2) \vu + (\mu_1-\lambda_1\fr r)\vn=0.
$$
This can be satisfied by selecting $\lambda_1=2$
and $\mu_1=2\fr r$. A similar argument for \eqref{eq:D2L}
shows that it is satisfied by the selection
$\lambda_2=2$ and $\mu_2=2\fr r$.
Thus the candidate solution is indeed a solution.
So if $\vq\notin\RR$, we have
$\diam(\RR)= \|2r\sqrt{1-\fr^2} \vu\|_2=2\sqrt{r^2-(\vn^T\vq-c)^2}.$

If $\vq\in \RR$, then for the candidate solution we take
$\vth_1=\vq+r\vu$ and $\vth_2=\vq-r\vu$. In this case, the candidate points satisfy the spherical bound with equality.
The half space constraint is also satisfied:
$\vn^T\vth_k= \vn^T\vq \pm r\vn^T\vu \leq c$, but generally only with inequality. So the candidate solution satisfies
all of the constraints and the first two complementary slackness conditions.
In this case, substitution of the expressions
for $\vth_1, \vth_2$ into \eqref{eq:D1L} yields
$$
r(\lambda_1-2) \vu + \mu_1\vn=0
$$
This can be satisfied by selecting $\lambda_1=2$
and $\mu_1=0$. This also satisfies the corresponding
complementary slackness condition. Similarly for \eqref{eq:D2L}.
So when $\vq\in \R$, $\diam(\RR)=\|2r\vu\|_2=2r$.
\end{proof}

\subsection*{Proof of Lemma \ref{lem:q}}
\begin{proof}
$\RR^0_k=S(\vq_k,r_k)\cap H_k$.
1) For $k=1$, $\vq_1 = \tv/\lambda_1, r_1=\|\tv/\lambda_1 - \tv/\lm\|_2$, $H_{1}=\{\vth\colon\at_*^T\vth\leq 1\}$. Since $\lambda_1 < \lm$, $\vq_1=\tv/\lambda_1$ is outside the dual feasible set $\FS$ and outside the half space $H_1$. Thus $\vq_1$ is outside $\RR^0_1$.
2) For $k=2, \ldots, N$, by construction $\RR^0_k \subset H_k=\{\vth: \vn_{k-1}^T \vth \leq c_{k-1} \}$, with $\vn_{k-1}$ and $c_{k-1}$ given by equation \eqref{eq:nck-1}. $\vq_{k-1}=\tv/\lambda_{k-1}\notin \FS$, since $\lambda_{k-1}<\lambda_1< \lm$. Since $\{\vth:\vn_{k-1}^T \vq_{k-1} =c_{k-1}\}$ is a separating hyperplane of $\vq_{k-1}$ from $\FS$, $\vn_{k-1}^T \vq_{k-1} >c_{k-1}$. Then $\vn_{k-1}^T \vq_k = \vn_{k-1}^T \vq_{k-1}*(\lambda_{k-1}/\lambda_k) >c_{k-1}$. Thus $\vq_k \notin \RR^0_k$.
\end{proof}

\subsection*{Proof of Proposition \ref{pro:diamRRk}}
\begin{proof}
By Lemma \ref{lem:diamR},
$\diam(\RR^0_k)=2\sqrt{r_k^2-(\vn_{k-1}^T\vq_k-c_{k-1})^2}$.
Set $s_k=1/\lambda_k$.
Then,
$$
r_k^2 =\|s_k\tv-\vtho_{k-1}\|_2^2
= \|(s_k-s_{k-1})\tv +s_{k-1}\tv-\vtho_{k-1}\|_2^2
= \|\Delta_k \tv +\beta_{k-1}\vn_{k-1}\|_2^2,
$$
where
$\Delta_k=s_k-s_{k-1}$ and
$\beta_{k-1}$ is a scalar such that
$\beta_{k-1}\vn_{k-1}= s_{k-1}\tv-\vtho_{k-1}.$
Hence
$$
r_k^2 = \Delta_k^2 \tv^T\tv +2\beta_{k-1} \Delta_k\tv^T\vn_{k-1}
+ \beta_{k-1}^2 .
$$
For the second term we have
\begin{align*}
(\vn_{k-1}^T\vq_k-c_{k-1})^2
&= (\vn_{k-1}^T(s_{k-1}\tv +s_k\tv -s_{k-1}\tv) -c_{k-1} )^2\\
&= (\vn_{k-1}^T(\Delta_k \tv +\beta_{k-1} \vn_{k-1}))^2\\
&=\Delta_k^2 \tv^T\vn_{k-1}\vn_{k-1}^T\tv
	+2\beta_{k-1}\Delta_k \tv^T\vn_{k-1} +\beta_{k-1}^2 .
\end{align*}
Thus
$\diam(\RR^0_k)= 2\Delta_k \sqrt{\tv^T(I-\vn_{k-1}\vn_{k-1}^T)\tv}$.
\end{proof}

\subsection*{Proof of Theorem \ref{thm:smlard}}

\begin{proof}
Under the DASS feedback rule,
\begin{equation}\label{eq:proofT1}
\textstyle
s_k=s_1+\nhalf R \sum_{j=2}^k
(\tv^T(I-\vn_{j-1}\vn_{j-1}^T)\tv)^{-\nhalf}
\geq s_1 +(k-1)\fracn{R}{(2\|\tv\|_2)} \ .
\end{equation}
Hence $s_1 +(k-1)\fracn{\nhalf R}{\|\tv\|_2} \geq 1/\lat$
is sufficient for termination,
from which we deduce
$N\leq 1+(\fracn{2\|\tv\|_2}{R})
(\fracn{1}{\lat}
-\fracn{1}{\lambda_1})$.
This bound is very loose.
Proposition \ref{pro:diamRRk} and the feedback rule \eqref{eq:updater} give $\diam(\RR_k)\leq R$.

The stronger bound on $N$ requires the follow more detailed argument.
%
%
Let the decreasing sequence $\{\lambda_k\}_{k=1}^N$ be
determined by \eqref{eq:updater} and $\alpha_k=\lambda_k/\lambda_{k-1}$.
Then
$$
\frac{1}{\lambda_k}-\frac{1}{\lambda_{k-1}}
=(\frac{1}{\alpha_k} -1)\frac{1}{\lambda_{k-1}}
= \nhalf R (\tv^T(I-\vn_{k-1}\vn_{k-1}^T)\tv)^{-\nhalf}
$$
So
$$
\alpha_k = \frac{1}{1+\nhalf R\lambda_{k-1}
(\tv^T (I-\vn_{k-1}\vn_{k-1}^T)\tv)^{-\nhalf}
}
$$
Using the definition of $\vn_{k-1}$, we have
$$
\tv^T\vn_{k-1}=\frac{\|\tv\|_2^2/\lambda_{k-1}
-\tv^T\vtho_{k-1}}{\|\tv/\lambda_{k-1}-\vtho_{k-1}\|_2}
= \frac{\|\tv\|_2^2 -\lambda_{k-1}\tv^T\vtho_{k-1}}
{\|\tv-\lambda_{k-1} \vtho_{k-1}\|_2}.
$$
Hence
\begin{align}
\tv^T(I-\vn_{k-1}\vn_{k-1}^T)\tv
&=\|\tv\|_2^2 - (\tv^T\vn_{k-1})^2\\
&= \|\tv\|_2^2
-\frac{
(\|\tv\|_2^2 -\lambda_{k-1}\tv^T\vtho_{k-1} )^2
}
{
\| \tv-\lambda_{k-1}\vtho_{k-1}\|_2^2
}\\
&=
\frac{\lambda_{k-1}^2(\|\tv\|_2^2\|\vtho_{k-1}\|_2^2 -(\tv^T\vtho_{k-1})^2)}
{
\| \tv-\lambda_{k-1}\vtho_{k-1}\|_2^2
}
\end{align}

Let $\vu_{k-1}=\vtho_{k-1}/\|\vtho_{k-1}\|_2$.
We first consider the possibility that $\vu_{k-1}=\beta \tv$
for some scalar $\beta\neq 0$. So $\tv$ and $\vu_{k-1}$ are aligned.
Since $\lambda_{k-1}<\lm$, in this case we must have $\vtho_{k-1}=\tv/\lm$.
Then $\vtho_k=\tv/\lm$ for all $k$. So this case is easily handled.
Hence we will assume below that $\vtho_{k-1}$ is not aligned with $\tv$.

We now make the assumption that there exists $C>0$ such that
$\|\vtho(\lambda)\|_2\leq C$ for $\lambda >0$.
Under this assumption we have:
\begin{align}
\lambda_{k-1} (\tv^T(I-\vn_{k-1}\vn_{k-1}^T)\tv)^{-\nhalf}
&=\left (\frac{
\|\tv-\lambda_{k-1}\vtho_{k-1}\|_2^2
}
{\|\tv\|_2^2\|\vtho_{k-1}\|_2^2 -(\tv^T\vtho_{k-1})^2} \right)^{\nhalf}\label{eq:preopt}\\
&\geq
\frac{1}{\|\vtho_{k-1}\|_2}
\left (\frac{\|\tv - \vu_{k-1}\vu_{k-1}^T\tv \|_2^2}
{\|\tv\|_2^2 -(\tv^T\vu_{k-1}) ^2} \right)^{\nhalf}
\label{eq:postopt}\\
&= \frac{1}{\|\vtho_{k-1}\|_2}\\
&\geq \frac{1}{C}.
\end{align}
Equation \eqref{eq:postopt} follows from \eqref{eq:preopt}
by optimizing over $\lambda_{k-1}$.
This gives an upper bound on $\alpha_k$
$$
\alpha_k
= \frac{1}{1+\nhalf R\lambda_{k-1}
(\tv^T (I-\vn_{k-1}\vn_{k-1}^T)\tv)^{-\nhalf}
}
\leq \frac{C}{C+\nhalf R}
<1.
$$
Set $\alpha = C/(C+\nhalf R)$. Then
$
\lat=\lambda_N = (\Pi_{j=2}^N \alpha_j) \lambda_1
$.
Hence
$
\log(\fracn{\lat}{\lambda_1}) \leq (N-1) \log\alpha
$.
Noting that $\lat/\lambda_1<1$ and $\alpha <1$,
this implies that
$$
N\leq 1+ \frac{\log(\lambda_1/\lat)}{\log(1/\alpha)}
\leq  1 + \frac{\log( 1/\lat)}{\log(1+ \nhalf R/C)}
$$

To derive this tighter bound we needed the existence
of $\alpha <1$ with $\alpha_k\leq \alpha$.
To show this, we assumed the dual regularization path is bounded.
\end{proof}

\begin{lemma}
If $\tv$ is in the range of $D$, then $\vtho(\lambda)$ is bounded.
\end{lemma}

\begin{proof}
By the optimality of $\wvo(\lambda)$, $\vzero$ must be in the subdifferential of
$\|\wvo(\lambda)\|_1$.
Hence $ D^T(\tv- D\wvo(\lambda) ) \in  \lambda \partial \|\wvo(\lambda)\|_1 $.
Rearrangement gives
$ D^T\vtho(\lambda) \in \partial \|\wvo(\lambda)\|_1 $.
Let $D=U\Sigma V^T$ be a compact SVD. Then
$U^T \vtho(\lambda) \in \Sigma^{-1} V^T \partial \|\wvo(\lambda)\|_1 $.
Since $\tv$ in the range of $D$, so is $\vtho(\lambda)$.
Hence $\vtho(\lambda) \in U \Sigma^{-1} V^T \partial \|\wvo(\lambda)\|_1 $.
The result follows by observing
that each $g \in \partial \|\wvo(\lambda)\|_1 $ is bounded.
\end{proof}

\subsection*{Proof of Corollary \ref{cor:smlard}}
\begin{proof}
1) Let ${\vtht}_k=\vtho_k+\verr_k$. The inequality \eqref{eq:proofT1} still holds. Hence DASS terminates in a finite number of steps.

2)
At step $k$, we have the solution returned ${\vtht}_{k-1}=\vtho_{k-1}+\verr_{k-1}$, with $\|{\vtho}_{k-1}-{\vtht}_{k-1}\|_2=\|\verr_{k-1}\|_2\leq \rho$ (assumption). As discussed earlier, generically $\vtho_{k-1}$ is not aligned with $\tv$, and for $k$ large, ${\vtht}_{k-1}$ is not aligned with $\tv$ as well. Let $\vnu_{k-1}={\vtht}_{k-1}/\|{\vtht}_{k-1}\|_2$. Substituting $\vtho_{k-1}$ with ${\vtht}_{k-1}$ in Equation \eqref{eq:preopt} for the error-free case, we have
\begin{align}
\lambda_{k-1} (\tv^T(I-\vn_{k-1}\vn_{k-1}^T)\tv)^{-\nhalf}
& =\left (\frac{
\|\tv-\lambda_{k-1}{\vtht}_{k-1}\|_2^2
}
{\|\tv\|_2^2\|{\vtht}_{k-1}\|_2^2 -(\tv^T{\vtht}_{k-1})^2} \right)^{\nhalf}\\
&\geq
\frac{1}{\|{\vtht}_{k-1}\|_2}
\left (\frac{\|\tv - \vnu_{k-1}\vnu_{k-1}^T\tv \|_2^2}
{\|\tv\|_2^2 -(\tv^T\vnu_{k-1}) ^2} \right)^{\nhalf}
\\
&= \frac{1}{\|{\vtht}_{k-1}\|_2}\\
&= \frac{1}{\|\vtho_{k-1}+\verr_{k-1}\|_2}\\
&\geq \frac{1}{C+\rho}.
\end{align}
Set $\alpha^\prime = (C+\rho)/(C+\rho+\nhalf R)$. It follows similarly,
$$
N\leq 1+ \frac{\log(\lambda_1/\lat)}{\log(1/\alpha^\prime)} \leq 1+ \frac{\log(1/\lat)}{\log(1/\alpha^\prime)}.
$$
It is easy to see $\alpha^\prime>\alpha$. So for inaccurate solutions, the convergence rate is still of order $O(\log(1/\lat))$, but has a larger upper bound.
\end{proof}


\subsection*{Effects of the choice of $R$}

We show below another empirical evaluation on the effect of $R$ using a new dataset.

\textbf{YALEBXF:}
This uses the frontal face images (size $192\times 168$) of
the 38 subjects  in the extended Yale B face recognition data set \cite{Georghiades2002From,Lee2005Acquiring}. We randomly select $m=2,000$ of the $2,414$ face images as the dictionary and randomly choose $\tv$ from the remaining $414$ images.\\

\begin{figure}[h]
\centering
\includegraphics[width=0.245\textwidth]{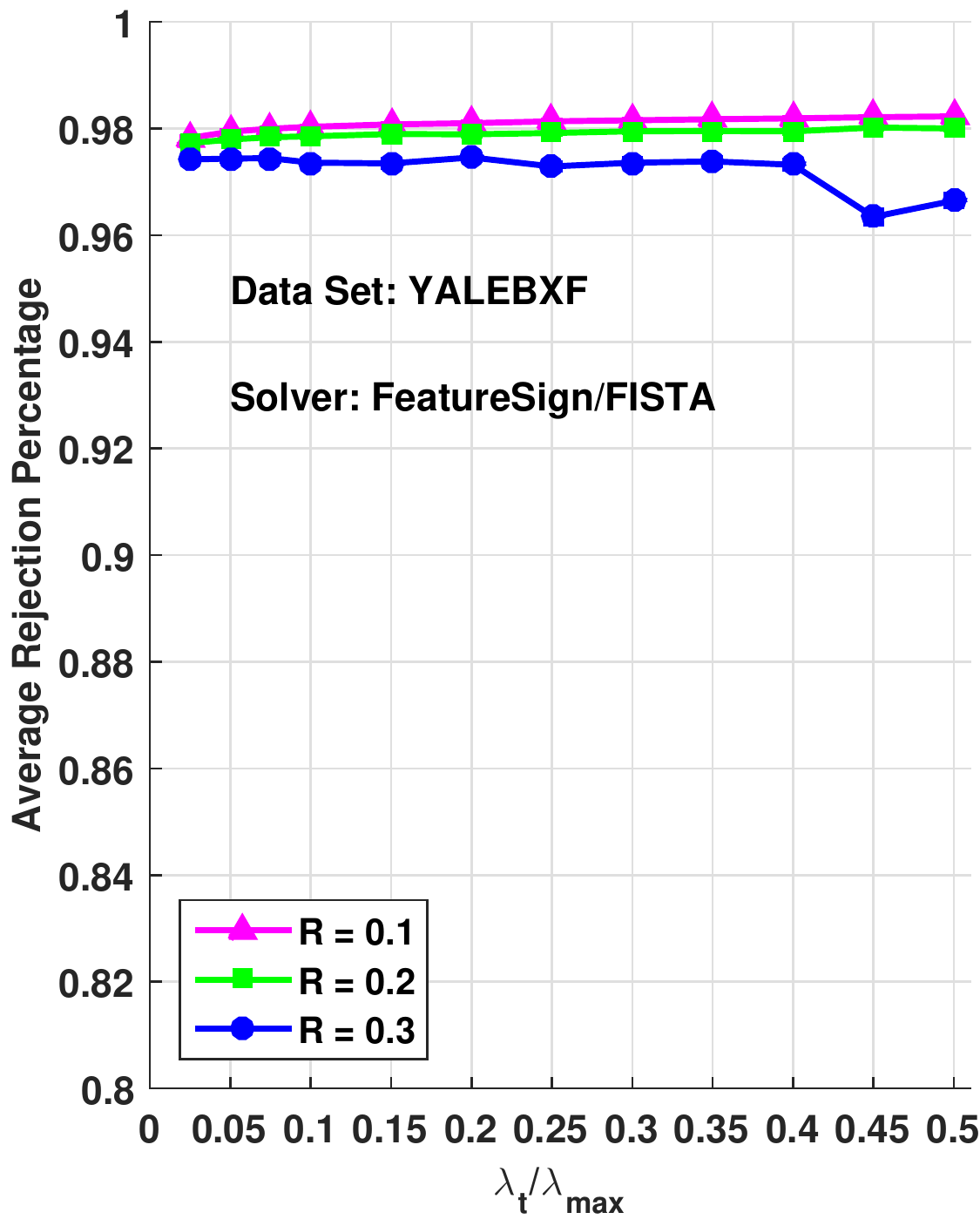}
\includegraphics[width=0.245\textwidth]{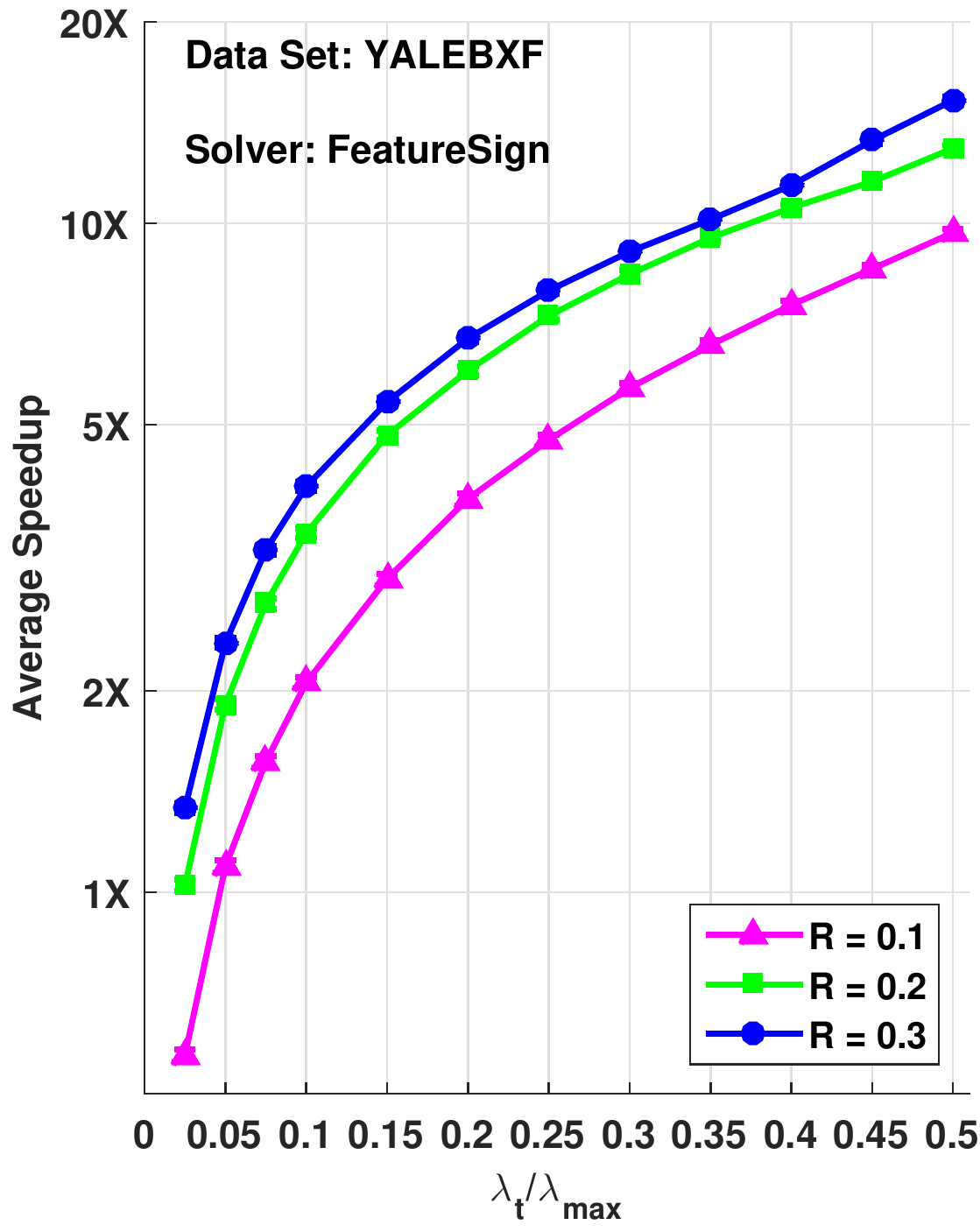}
\includegraphics[width=0.245\textwidth]{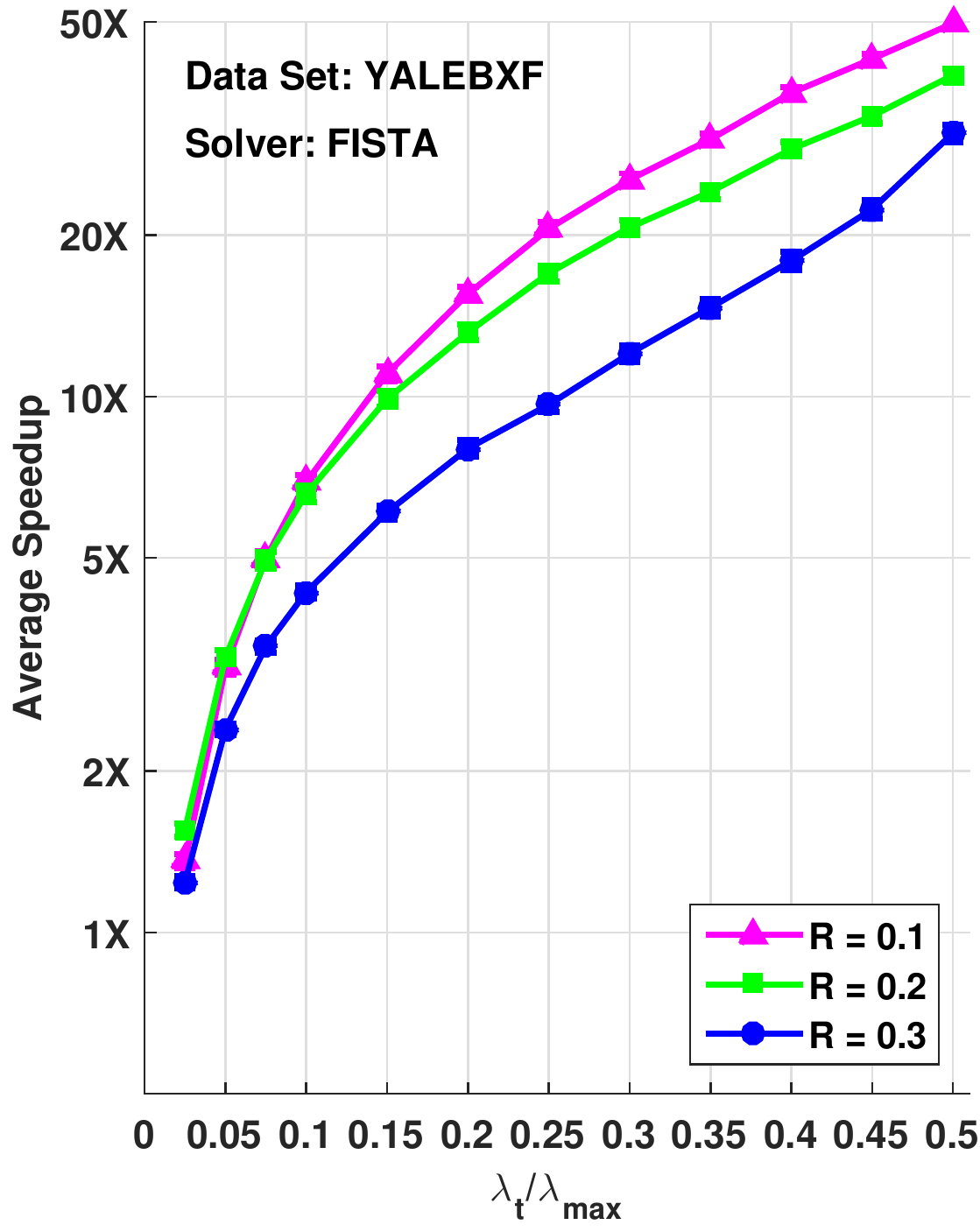}
\includegraphics[width=0.245\textwidth]{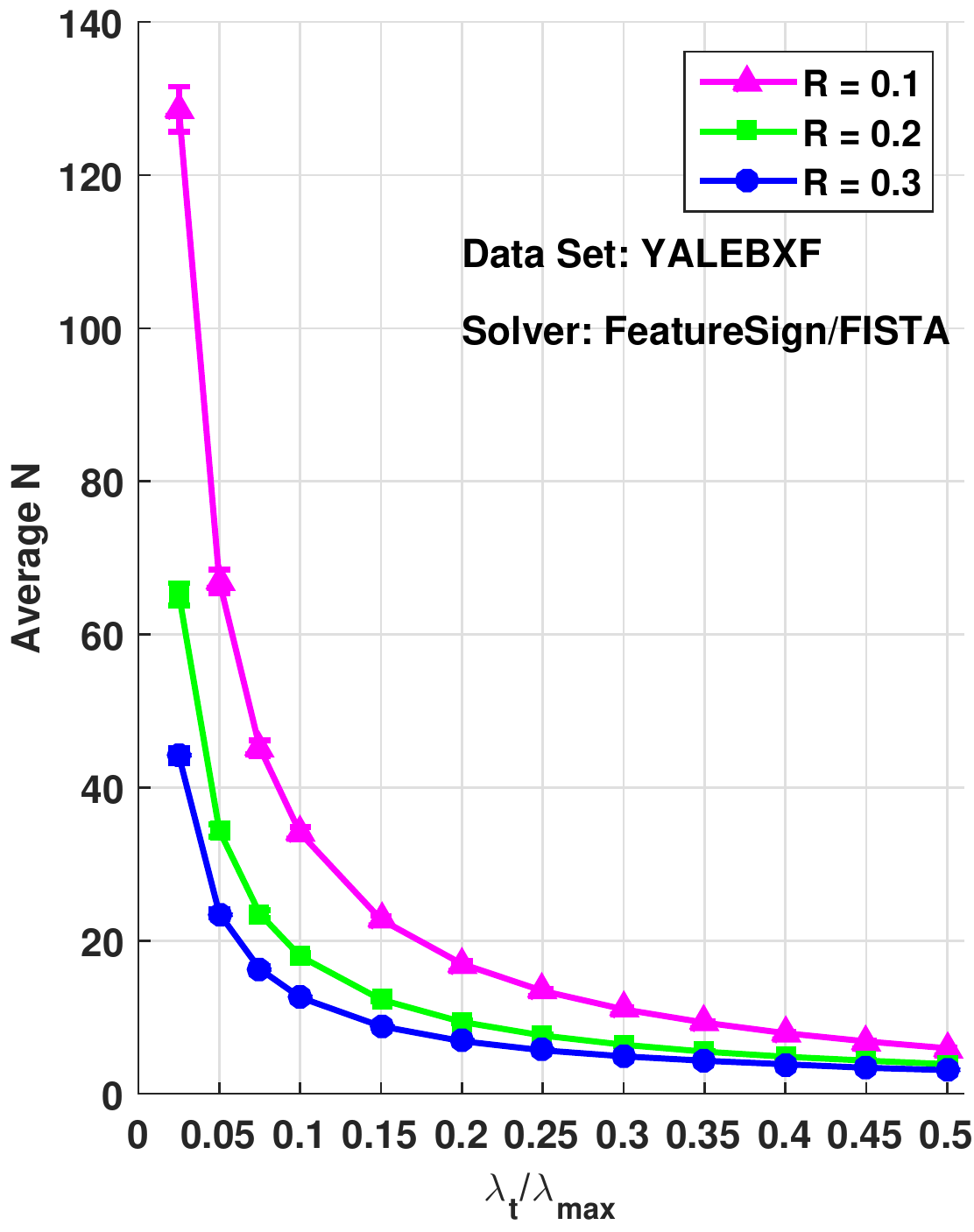}
\caption{
\small {\bf Effects of the choice of $R$ on DASS performance.}
(Left): Average rejection percentage.
(Middle): Average speedup.
(Right): The average value of $N$ over the $1200$ trials.
}
\label{fig:R2}
\end{figure}

The trend in FISTA \cite{AM2009} result is similar to that in FeatureSign \cite{Lee2007Efficient}. We observed that for the datasets we tested, FeatureSign seems to be faster than FISTA, and FISTA is more sensitive to the reduction of features when used in conjunction with a screening rule and thus has greater speedup.
\subsection*{A Feeckback rule for sequential DPP \cite{dpp2015}}
We can extend our feedback idea to the sequential DPP rule proposed in \cite{dpp2015}. Sequential DPP uses the spherical bound $S^{\text{dpp}}_k=\{\vth: \|\vth -\vtho_{k-1}\|_2
\leq 1/\lambda_k - 1/\lambda_{k-1}\} $ to one-shot screen the dictionary at step $k$. The feedback rule for Sequential DPP is $\lambda_k =(1/2R+\lambda^{-1}_{k-1})^{-1}$. Under this rule, it follows that for all $(\tv,\lat)$, the sequence terminates after a finite number of steps $N \leq 1+\frac{2(1/\lat-1/\lambda_1)}{R} \leq 1+\frac{2/\lat}{R}$, and $\diam(S^{\text{dpp}}_k) \leq R$, $k=2\mc N$. Compared with DASS, feedback controlled sequential DPP has a larger upper bound on the number of iterations required.

\end{document}